\documentclass[twoside]{article}

\usepackage[english]{babel}
\usepackage{natbib}
\usepackage{amsmath,amssymb,amsthm,graphicx,algorithm,algpseudocode,enumitem,xcolor,enumitem,mathrsfs}
\usepackage{subcaption}
\usepackage{amsthm}
\usepackage{dsfont}
\usepackage{booktabs}
\usepackage{hyperref}
\usepackage{bm}
\hypersetup{
    colorlinks=true,
    linkcolor=black,
    citecolor=cyan,
    filecolor=black,
    urlcolor=black
}
\theoremstyle{plain}
\newtheorem{theorem}{Theorem}[section]
\newtheorem{lemma}{Lemma}[section]

\theoremstyle{remark}

\newtheoremstyle{bolddefinition}%
  {3pt}%
  {3pt}%
  {\itshape}%
  {}%
  {\bfseries\boldmath}%
  {.}%
  {.5em}%
  {\thmname{#1}\thmnumber{ #2}\textbf{\thmnote{ [#3]}}}
\theoremstyle{bolddefinition}
\newtheorem{definition}{Definition}[section]
\newtheorem{assumption}[definition]{Assumption}
\newtheorem{proposition}[definition]{Proposition}

\newcommand{\gapi}[1]{\Delta_{\sigma_{#1}, \sigma_{#1+1}}}
\newcommand{\arm}[1]{\mu_{\sigma_{#1}}}

\newcommand{\Sone}{\mathcal{S}^1}
\newcommand{\Sp}{\mathcal{S}^p}
\newcommand{\Egood}{\mathcal{E}_{\text{good}}}
\newcommand{\Carm}{\hat{\nu}^*}

\usepackage[preprint]{aistats2026}

\begin{document}

\twocolumn[

\aistatstitle{Meet Me at the Arm: The Cooperative Multi-Armed Bandits Problem with Shareable Arms}

\aistatsauthor{  Xinyi Hu \And Aldo Pacchiano }

\aistatsaddress{ Boston University \And  Boston University and  Broad Institute of MIT and Harvard } ]

\begin{abstract}
  We study the decentralized multi-player multi-armed bandits (MMAB) problem under a no-sensing setting, where each player receives only their own reward and obtains no information about collisions. Each arm has an unknown capacity, and if the number of players pulling an arm exceeds its capacity, all players involved receive zero reward. This setting generalizes the classical unit-capacity model and introduces new challenges in coordination and capacity discovery under severe feedback limitations. We propose \textsc{A-CAPELLA} (Algorithm for Capacity-Aware Parallel Elimination for Learning and Allocation), a decentralized learning algorithm that achieves logarithmic regret in this generalized regime via protocol-driven coordination. %
\end{abstract}

\section{INTRODUCTION}

\label{sec:introduction}
The multi-armed bandit (MAB) is a foundational framework in sequential decision-making that captures the trade-off between \textit{exploration} and \textit{exploitation}. In the classical setting, a learner is presented with \( K \) arms, each associated with an unknown reward distribution. The goal is to identify and exploit the best arm over a time horizon. At each round \( t = 1, \dots, T \), the learner selects an arm \( a_t \in [K] \) and receives a stochastic reward drawn from the corresponding distribution.

In this paper, we study the \emph{collaborative multiplayer multi-armed bandit (MMAB)} problem, where multiple players interact with a shared set of arms. Unlike prior MMAB settings which assume unit capacity--i.e., each arm can support at most one player at a time, and any collision results in zero reward for all involved players~\citep{boursier2019b, wang2020, lugosi2022}--we generalize the problem by allowing each arm to have an unknown capacity possibly greater than  one.

We propose the Multi-Player Multi-Armed Bandits with Shareable Arms, local observations, and hard overload penalties (\textbf{MMAB-SAX})  framework, where \( M \) decentralized players interact with \( K \) arms. The instance is defined by a mean reward vector \( \boldsymbol{\mu} = (\mu_1, \dots, \mu_K) \in [0,1]^K \), where each arm \( \nu \in [K] \) has an unknown mean reward \( \mu_\nu \) and an unknown capacity \( C_\nu \in \mathbb{N} \). The capacity \( C_\nu \) represents the maximum number of players that can simultaneously pull arm \( \nu \) without incurring a collision penalty. At each time step \( t = 1, \dots, T \), each player \( p \in [M] \) independently selects an arm \( a_t^p \in [K] \). If the number of players selecting arm \( \nu \) exceeds its capacity \( C_\nu \), then all players on that arm receive zero reward. Otherwise, each player pulling that arm receives an independent stochastic reward with expectation \( \mu_\nu \).

Each player observes only their own reward and receives no information about the number of players on the same arm or whether a zero reward was caused by a stochastic realization or a collision. This strict no-sensing feedback and decentralized observation regime poses significant challenges for coordination and capacity discovery.

Our model captures practical real-world applications in decentralized IoT network systems, such as unlicensed spectrum access in TV White Space and RF backscatter communication, where arms (communication channel) naturally have capacity greater than one, and devices operate independently under no sensing constraints, each observing only their own reward without information about collisions or other users \citep{foukalas2014, matsumura2015,elias2015a,hessar2018, gupta2023, gong2024}.

The goal is to design a learning algorithm that maximizes the cumulative reward over time while observation remains decentralized, and explicit runtime communication is infeasible. However, we assume that agents agree beforehand on a shared protocol governing action selection and signal interpretation. Under this assumption, our algorithm leverages structured collision patterns as implicit binary signals to enable coordination and synchronization.

We evaluate our algorithm using the following notion of pseudo-regret:
\begin{align}
\mathcal{R}_T
&= T \cdot 
   \max_{\substack{
      \mathbf{a} \in [\mathbb{N} \cup \{0\}]^K \\
      \sum_{i=1}^K a(i) = M,\; a(i) \le C_i
   }}
   \langle \mathbf{a}, \boldsymbol{\mu} \rangle \notag \\
&\quad - \sum_{t=1}^T \sum_{\nu=1}^K
\mu_{\nu} \cdot
\,\psi_t(\nu)
\,\mathds{1}\!\bigl\{
    \psi_t(\nu) \le C_{\nu}
\bigr\},
\end{align}
where $\psi_{t} (\nu)= \sum_{p=1}^M\mathds{1}\{a_t^p = \nu\}$, is the number of players pulling arm $\nu$ at time $t$.

For simplicity of presentation, we impose the following assumption on the arm reward distribution: 
\begin{assumption}\label{assumption:distinct-means}
  The reward distributions of all arms are supported on \([0, 1]\), and for any \(i \neq j\), we have \(\mu_i \neq \mu_j\).
\end{assumption}

We use the notation \( \sigma \in \mathbb{S}_K \) to denote the unknown permutation such that \( \mu_{\sigma(1)} > \mu_{\sigma(2)} > \cdots > \mu_{\sigma(K)} \). For any \( i,j \in [K] \), we denote the gap between the \( i \)-th and \( j \)-th best arms as:
\(\Delta_{\sigma_i, \sigma_j} = \mu_{\sigma_i} - \mu_{\sigma_j}\), where \(i < j\). Assumption~\ref{assumption:distinct-means} can be relaxed, as discussed in Appendix~\ref{sec:relax1.1}.

\section{RELATED WORK}

Early interest in multiplayer bandits grew out of \textit{cognitive radio networks}, where radios must learn to share spectrum opportunistically \citep{mitola1999}. Some early works applied multiplayer multi-armed bandit formulations to model how cognitive radios should select channels, typically assuming unit capacity on each arm\citep{jouini2009,jouini2010,liu2008,liu2010,anandkumar2011}. Most of the literature assumes that all players know the time horizon $T$; however, this assumption can sometimes be relaxed \citep{degenne2016}. In this work we adopt the standard convention of a known horizon.

Decentralized multiplayer bandit models typically fall into two categories: \textit{sensing} and \textit{no-sensing}. In sensing models, players can detect collisions. Notably, the Musical Chairs routine introduced by \citep{rosenski2016} provides fast orthogonalization and has been used in several algorithms. Other approaches such as SIC-MMAB and DPE1 exploit collision feedback to transmit bits \citep{boursier2019b,wang2020,lugosi2022}. DPE1 achieved regret that matches the centralized optimal performance \citep{wang2020}. 

In contrast, no-sensing models hide collision information: players only observe their own reward, which may be zero due to either collisions or stochasticity. Despite this challenge, several recent works have achieved logarithmic regret in this setting without introducing additional parameters, by leveraging communication through a high-reward arm \citep{huang2022b,pacchiano2023}. Our work adopts a similar idea. Furthermore, it has been shown that \textit{instance-dependent} regret scaling as \( \tilde{O}(1/\Delta) \), where \( \Delta \) is the gap between the optimal arms and the remaining arms, is impossible without some form of communication \citep{liu2022a}. There also exist algorithms such as the Selfish algorithm \citep{besson18}, where each player simply runs UCB independently without coordination; however, such approaches lack theoretical guarantees on the upper bound. To the best of our knowledge, lower bounds under the decentralized no-sensing setting remain open. 

Recent work has extended multiplayer bandits to the setting of \textit{shareable-capacity arms} (MMAB-SA), where each arm can support multiple players up to an unknown capacity. \citet{wang2022,wang2022a} introduced models in both centralized and decentralized settings where each arm has an unknown capacity. In their setting, if the number of players exceeds an arm's capacity, only the overloaded players receive zero reward.  In the centralized setting, follow-up work by \citet{li2024} tightened the regret lower bound established by \citet{wang2022} and proposed a centralized algorithm that matches this minimax lower bound. 
 
In the decentralized setting, \citet{wang2022a} introduce two feedback mechanisms: Sharing Demand Awareness (SDA), where a player receives a binary indicator of whether the arm is shared, and Sharing Demand Information (SDI), where the exact number of co-players is observed. Under different feedback schema, their algorithms, DPE-SDI and SIC-SDA, strategically induce overloads to estimate arm capacities. However, the instance-dependent regret of these algorithms grows as $\mathcal{O}(\log(T) \mu_K^{-2})$ depending on worst arm. In contrast, our algorithm analysis delivers a sharper instance-dependent upper bound that depends only on $\min(\Delta_{\sigma_V,\sigma_{V+1}}, \Delta_{\sigma_{V-1},\sigma_V})^{-1}$, where \(V\) is the minimum number of top arms needed to accommodate all players. When there is a large separation between the top \(V\) arms and the remaining arms, our algorithm quickly identifies the optimal arm set.

Table~\ref{tab:comparison} contrasts our MMAB-SAX model with the decentralized shareable-arm setting of Wang et al.\ (2022). While both assume unknown arm capacities, Wang et al.\ provide \emph{arm-sharing feedback} facilitating implicit communication. In MMAB-SAX, by contrast,  no collision or sharing feedback is available. This stricter regime rules out existing methods such as SIC-MMAB2,  an adapted version of SIC-MMAB under no-sensing, and \citet{wang2022a}'s, which rely on unit-capacity assumptions or sharing feedback as a communication channel. Since instance-dependent \( \tilde{O}(1/\Delta) \) bounds are impossible without communication, we introduce \textsc{A-CAPELLA}, which achieves logarithmic instance-dependent \( \tilde{O}(1/\Delta) \) regret under this fully no-sensing model when capacities are unknown.

\begin{table}[h]
\caption{Comparison of Model Assumptions}
\label{tab:comparison}
\begin{center}
\begin{tabular}{lp{2.7cm}p{2.7cm}}
\textbf{Aspect} & \textbf{Wang et al.} & \textbf{MMAB-SAX} \\
\toprule
Collision & Overloaded players get $0$ 
          & Players on overloaded arm get $0$ \\
          \hline
Reward    & Total per arm reward observed  
          & Own reward observed \\
          \hline
Capacity  & Unknown 
          & Unknown \\
          \hline
Sensing   & SDA/SDI 
          & No sensing \\
\end{tabular}
\end{center}
\end{table}

\section{ASSUMPTIONS \& NOTATIONS}
\label{sec:assumptions_notations}
We generalize prior work that assumes \( C_{\nu} = 1 \) by allowing each arm \( \nu \in [K] \) to have capacity \( C_{\nu} \ge 1 \). Each arm has a reward distribution \( \mathcal{D}_{\nu} \) supported on \( [0,1] \), with mean \( \mu_{\nu} \). If at most \( C_{\nu} \) players pull arm \( \nu \), each receives an independent sample from \( \mathcal{D}_{\nu} \). If more than \( C_{\nu} \) players pull \( \nu \), a collision occurs and all players on that arm receive a reward of zero.

The capacity-aware reward distribution is:
\[
\mathcal{D}_{\nu,\psi} = 
\begin{cases}
\mathcal{D}_{\nu} & \text{if } \psi \le C_{\nu}, \\
0 & \text{if } \psi > C_{\nu},
\end{cases}
\quad \text{with} \quad
\mu_{\nu,\psi} = \mathbb{E}[\mathcal{D}_{\nu,\psi}],
\]

where $\psi$ denotes the number of players simultaneously pulling arm $\nu$.

\begin{assumption}[Shared Knowledge]
    \label{shared knowledge}
    We assume that all players and arms are indexed from 1 to \(M\) and \(K\) respectively, and that all players know their own index number and the index of each arm.  The functions \(B(\cdot),\omega(\cdot)\), and the constant \(h\) are also known to all players at the beginning.
\end{assumption}

\begin{definition}[Simple Round Robin Scheduling]
In this protocol, each player \( p \in \{1, \dots, M\} \) begins by pulling arm \( p \) in the first round. In round \( i \geq 1 \), player \( p \) pulls arm 
\(
(p + i-1) \bmod K.
\) Arm 0 is to be interpreted as arm $K$.
\end{definition}
For simplicity, assume \( K \geq M \), and this assumption  can be relaxed with minor modifications to the algorithm, while the instance-dependent \( \tilde{O}(1/\Delta) \) regret guarantees remain unchanged (see Appendix~\ref{sec: M K assumption}). A complete round of Simple Round Robin occurs when all players have pulled every arm once. This scheduling ensures systematic exploration of all arms without collisions. %
We denote \(\text{RR}([K],w)\) as \(w\) consecutive pulls under Simple Round Robin Scheduling on \([K]\) arms.

\vspace{0.2cm}
\begin{definition}[\(\psi\)-Grouped Round Robin Scheduling]
In this procedure, players are divided into groups of size \( \psi \) for \(\psi\in [M]\). If the total number of players \(M\) is not divisible by \(\psi\), the remaining players form a smaller group. Each group then performs a standard Round Robin exploration over all arms as a single unit. To ensure complete coverage, the procedure is repeated twice. During the second round, players who were previously in a smaller group are prioritized to be placed in a full group of size \(\psi\). This repetition is necessary because without it, some players would have less information about the capacity of arms when \(\psi\) players pull them simultaneously.
\end{definition}
Figure~\ref*{fig:2grouped_rr} illustrates an example of 2-Grouped Round Robin with 3 players and 5 arms. A complete \emph{Grouped Round Robin Session} consists of sequentially executing \( \psi \)-Grouped Round Robin for each \( \psi = 1 \) to \(  m \), where \( m \) denotes the number of \emph{active players} who are still exploring arms in the current epoch. Note that Simple Round Robin is equivalent to 1-Grouped Round Robin.

\begin{figure}[ht]
  \centering
  \includegraphics[width=1\linewidth]{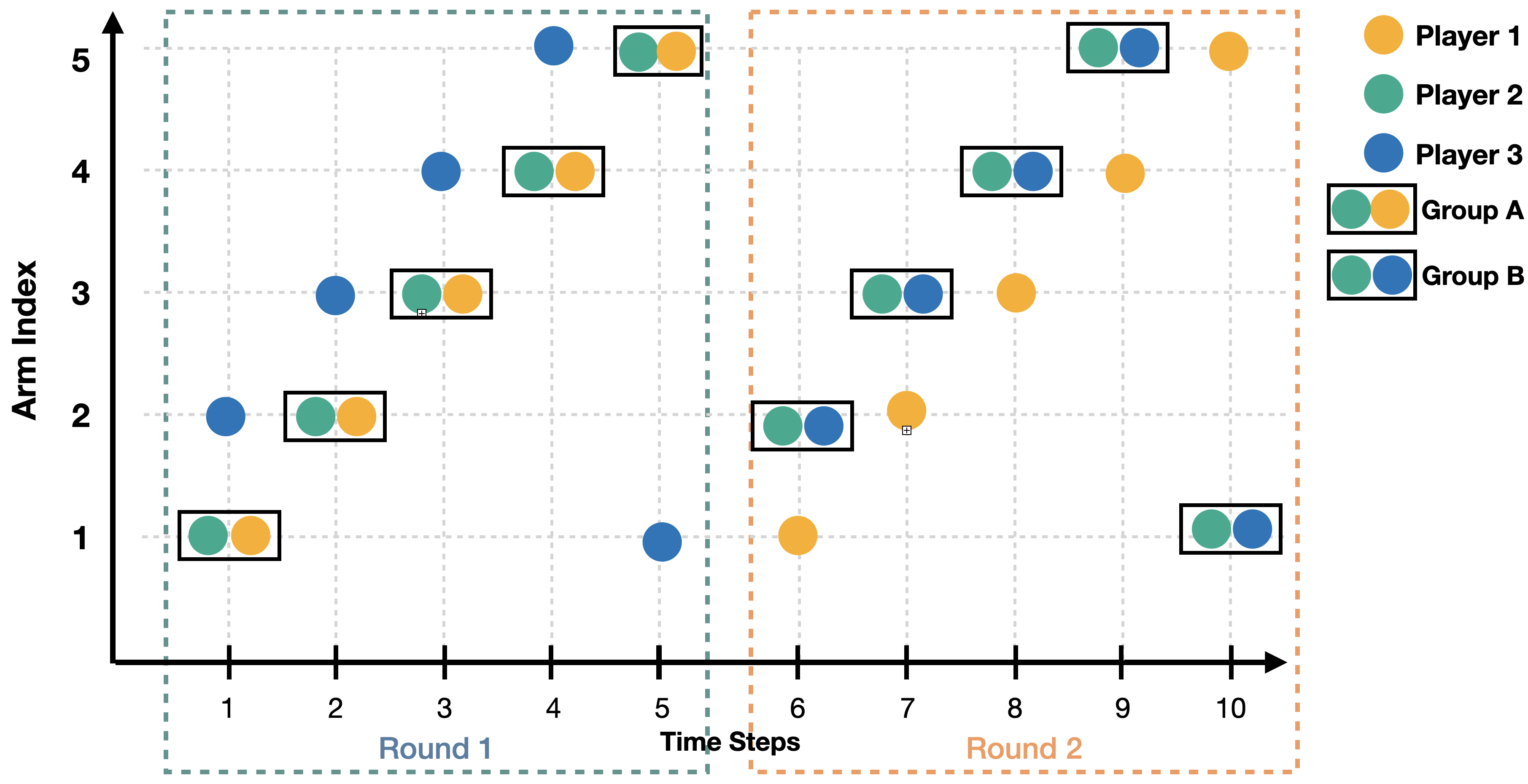}
  \vspace{.1in}
  \caption{
    Illustration of a 2-Grouped Round Robin Scheduling Strategy For 3 Players and 5 Arms. 
  }
  \label{fig:2grouped_rr}
\end{figure}

  Define \( N_{t}^p(\nu,\psi) \) as the total number of times player \( p \) has pulled arm \( \nu \) during \(\psi\)-Grouped Round Robin rounds, where she was successfully assigned to a group of size \( \psi \), up to time \( t \). Let \( N_t(\nu, \psi) \) denote the minimum of \( N_t^p(\nu, \psi) \) across all players by time \( t \), i.e., 
  \(
  N_t(\nu, \psi) = \min_{p} N_t^p(\nu, \psi).
  \)
  Furthermore, we use \(  N_t(\
  \cdot, \psi) \) to denote the minimum \(N_t(\nu, \psi)\) across all active arms by time \( t \).
  Under deterministic Round Robin Scheduling, each player can compute these values locally without communication.

  Based on these observations, each player independently constructs confidence intervals for their estimates of the mean rewards. To formalize this, let \( a_\ell^p \), \( r_\ell^p \), and \( \psi_\ell \) respectively denote the arm pulled by player \( p \), the corresponding reward received, and the group size used in the Round Robin at time \( \ell \).
  Then, the estimated mean reward for arm \( \nu \) by player \( p \) up to time \( t \) during \(\psi\)-Grouped Round Robin rounds is given by:
  \[
  \hat{\mu}_{\nu,\psi}^p(t) = \frac{1}{N_t^p(\nu,\psi)}\sum_{\ell=0}^t r_\ell\cdot  \mathds{1}\bigl\{a_\ell^p = \nu, \psi_\ell = \psi\bigl\},
  \] 
where \(\mathds{1}(\cdot)\) is the indicator function. %
 Choosing a \(\delta \in (0, 1)\) to be the probability parameter, the confidence interval is constructed using the confidence radius defined as 
\(B:\mathbb{N} \to \mathbb{R}^+,  \)
\begin{align} 
B(n) &= \sqrt{\tfrac{2g(n)}{n}}, 
\quad g(n) = \log\!\left(\tfrac{4n^2 M K}{\delta}\right).
\end{align}

After applying Hoeffding’s inequality, for any fixed $p \in [M]$ and $\nu \in [K]$, with probability at least $1-\delta/MK$, we have: 
\begin{equation}
  \left| \hat{\mu}_{\nu,\psi}^p(t) - \mu_{\nu,\psi} \right| \leq B\Bigl(N_t^p(\nu,\psi)\Bigr)
  \label{eq:confidence_bound}
\end{equation}
for all \(t\in \mathbb{N}\) simultaneously. As it is customary our analysis will take place in the ``good event'' when all the confidence intervals are valid:

\begin{proposition}[The Good Event \( \Egood \)]\label{proposition:good_event}
The good event \( \Egood \) occurs when all confidence intervals defined in Equation~\ref{eq:confidence_bound} hold simultaneously for every player \( p \in [M] \), every arm \( \nu \in [K] \), and every time step \( t \in \mathbb{N} \). %
\end{proposition}
A union bound implies that $\Egood$ holds with probability at least \(1-\delta\).

\section{ALGORITHM OVERVIEW}
\label{sec:algorithm}
Our algorithm operates in three sequential phases, with player 1 designated as the \emph{coordinator}. The phases are designed to progressively identify and exploit the optimal set of arms while maintaining decentralized coordination through implicit signaling.

In Phase 1, the coordinator identifies a common communication arm, which serves as the coordination mechanism for subsequent phases. To inform the other players, known as \emph{listeners}, of the selected arm, the coordinator deliberately disrupts the dynamics of a high-reward arm--one that is comparable to the best arm--thereby creating a detectable signal. 
In Phase 2, the coordinator constructs a connectivity graph of arms by connecting arms whose empirical confidence intervals overlap, while concurrently executing a simple round robin exploration and updating the connectivity graph. This graph naturally partitions into two components: a top cluster comprising arms with higher estimated mean rewards, and a bottom cluster consisting of arms with lower estimated rewards. After the clustering is established, the coordinator signals the start of communication and informs the other players of the cluster containing more promising arms.
Phase 3 tests the capacities of arms in the high-reward cluster and marks those with known capacities.

The algorithm recursively executes Phase 2, conditionally skipping Phase 3 if the capacities of all arms in the high-reward cluster identified during Phase 2 already have known capacities. This process continues until the algorithm identifies the optimal \(V\) arms, where \(V\) is defined as the integer satisfying:
\vspace{-0.5em}
\begin{equation}
\sum_{i=1}^{V-1} C_{\sigma_i} < M \leq \sum_{i=1}^V C_{\sigma_i}.
\label{eq:V_arms}
\end{equation}

Our main result is the following theorem.
\begin{theorem}[\textbf{Simplified}]
\label{thm:informal_main}
There exists a decentralized algorithm which, with probability at least \( 1 - \delta \), achieves a pseudo-regret of order
\[
O\Bigl( \frac{\mathrm{poly}(M, K)}{\min(\Delta_{V-1, V},\Delta_{V, V+1})} \cdot \log\left( \frac{1}{\delta} \right) \Bigr)
\]
\end{theorem}

\subsection{Phase 1: Communication Arm Identification} 
\label{sec:Phase1}
In this section, we first explain the intuition behind Phase 1, then formally define the protocol for the coordinator in Section \ref*{sec: phase 1 coordinator} and for listeners in Section~\ref*{sec: phase 1 listener}.

The phase begins with all players repeatedly executing full sessions of Grouped Round Robin, where the group size $\psi$ ranges from $1$ to $M$. This process is used to estimate mean rewards and explore the capacity $C_\nu$ of some ``good'' arms with sufficient accuracy. Once the coordinator identifies an arm with a high mean reward that is suitable for communication (denoted $\Carm$), she begins signaling by repeatedly playing $\Carm$ alone over several rounds, thereby disrupting its max capacity. Listeners remain in a passive state advancing the grouped round robin schedule, waiting to detect a signal. If a listener observes that an arm, which is believed to be a likely communication candidate, shows an obvious drop in its estimated maximum capacity $C_\nu$, they infer that this arm is $\Carm$. At the end of Phase 1, all players will agree on the communication arm $\Carm$, and the number of consecutive zero rewards required to send one bit.

A key challenge in this procedure is that players may identify different communication candidates at different times, causing them to begin signaling or listening at unsynchronized stages. For example, the coordinator might start signaling before others are ready to receive. To resolve this issue, all players will build a confidence interval using an \textit{Inflated Confidence Radius} and use the collection of pre-specified, problem independent checkpoints \(\mathcal{T}_{\text{test}}\) defined below to synchronize the start of the communication phase. This helps ensure that players begin communicating at coordinated times.

We define checkpoints \(\mathcal{T}_{\text{test}}\) as time steps where the base-9 floor function value changes:
\(
\mathcal{T}_{\text{test}} = \left\{ t \mid \left\lfloor \frac{2}{B(N_t(\cdot,1))^2} \right\rfloor_9 > \left\lfloor \frac{2}{B(N_{t-1}(\cdot,1))^2} \right\rfloor_9 \right\}
\)
, where for any \(x
\in\mathbb{R}^+\), the base-9 floor function is defined as:
\(
\left\lfloor x \right\rfloor_9 = 9^{\alpha} \text{ for the unique } \alpha \in \mathbb{N} \text{ satisfying } 9^{\alpha} \leq x < 9^{\alpha+1}
\). Note that although the definition involves \( N_t(\cdot,1) \), this value is identical for all arms at the end of each Grouped Round Robin session. As a result, \( \mathcal{T}_{\text{test}} \) is a global checkpoint set shared across all arms.

\paragraph{Inflated Confidence Radius}
    Let \(h\in\mathbb{R}^+\) and \(h>2\), the term \( h \cdot B(N_t(\nu,1)) \) denotes an inflated confidence bound around the estimated mean reward \( \hat{\mu}_{\nu,1}(t) \) for arm \( \nu \), computed from simple (1-grouped) Round Robin observations up to time \( t \). The inflation factor \( h = 5 \) is chosen to ensure that the inequality 
    \begin{equation}
    \hat{\mu}_{\nu,1}^p - 5B(N_t(\nu,1)) \geq 0
    \label{eq:inflated radius greater than mu_hat}
\end{equation}
    is satisfied at approximately the same time for all players. It can be shown that this event occurs when   
    \(
      \frac{\mu_{\nu}}{12} < B(N_t(\nu,1)) \leq \frac{\mu_{\nu}}{4} \), and 
\( 2 / B(N_t(\nu,1))^2 \) falls within an interval that aligns with a unique power of 9 (see Lemma \ref{lem:confInt_bound}). Specifically, the triggering condition satisfies:
    \begin{equation}
        \label{eq:inflated radius greater than mu_hat interval}
        \frac{32}{\mu_{\nu}^2} \leq \frac{g(N_t(\nu,1))}{N_t(\nu,1)} < \frac{288}{\mu_{\nu}^2},
        \end{equation}
   \vspace{-0.02em}     
     and there exists a unique \(\alpha \in\mathbb{N}^+\) such that \(9^\alpha \in [\frac{32}{\mu_{\nu}^2} , \frac{288}{\mu_{\nu}^2})\). This interval structure ensures temporal synchronization of confidence-threshold crossings across players.

\subsubsection{Coordinator Protocol (Player 1)}
\label{sec: phase 1 coordinator}

\begin{algorithm}
        \caption{Phase 1: Coordinator}
        \label{alg:coordinator_protocol_phase1}
        \begin{algorithmic}[1]
        \State Initialize \( \mathcal{S}^1 \gets \emptyset \)
        \While{\(\mathcal{S}^1 = \emptyset\)}
        \State Follow \textsc{GroupedRoundRobin} (GRR)
            \If{ \( \exists \nu \text{ s.t. }5B(N_t(\nu,1)) > \hat{\mu}_{\nu,1}^1(t) \) }
                \State Add \( \nu \) to \( \mathcal{S}^1 \)
                \State Add all \( \nu' \) such that satisfies Eq.(\ref{eq:s1_arms}) to \( \mathcal{S}^1 \)
      
            \EndIf
        \EndWhile
      
        \State Find next checkpoint \( t' \in \mathcal{T}_{\text{test}} \) such that \( t' \geq t \)
      
        \While{\( t<t'\)}
            \State Follow \textsc{GRR} protocol
        \EndWhile
        \If{ \( \exists \nu \in \mathcal{S}^1 \) such that \( C_\nu < M \) }

   \For{$j = 1$ to $\omega(12,t_{\text{comm}})$}
    \State Pull $\Carm$ for the entire duration of one \textsc{GRR} session
    
\EndFor
\State \Return $\Carm,\; \omega(12,t_{\text{comm}})$

      \Else
          \State Switch to playing UCB algorithm on \(\mathcal{S}^1 \) until the end.
      \EndIf
      
        \end{algorithmic}
        \end{algorithm}
 While playing Grouped Round Robin Session scheduling, the coordinator (Player~1) independently estimates the mean reward \( \hat{\mu}_{\nu,1}^1(t) \) of each arm \( \nu \) using solo (\(\psi=1\)) observations. The set of candidate communication arms \( \mathcal{S}^1 \) is initialized as empty. For the first arm \( \nu \) that satisfies Equation \ref*{eq:inflated radius greater than mu_hat}, the arm \( \nu \) is added to \( \mathcal{S}^1 \), and we denote that time step as \(t_0^1\). In addition, all other arms \( \nu' \) satisfying 

\begin{equation}
\label{eq:s1_arms}
    \hat{\mu}_{\nu',1}^1(t_0^1) \geq \frac{3}{5} \hat{\mu}_{\nu,1}^1(t_0^1)
\end{equation} 
are also included in the set. Under the good event \(\Egood\), Lemma~\ref*{lemma:Sone subset Sp} ensures that the optimal arm \(\sigma_1\) is included in \(\mathcal{S}^1\). The coordinator then waits for the appropriate timing to begin signaling the communication arm selected from \( \mathcal{S}^1 \).

To ensure that all non-coordinator players are in a listening phase, starting at time \( t_0^1 \), Player~1 identifies the first checkpoint \( t' \) in the collection \(\mathcal{T}_{\text{test}}\) that occurs after or at \( t_0^1 \). The coordinator then waits for one round, skips \(t'\), and begins signaling at the next checkpoint, which we label as \( t_{\text{comm}} \). Since players operate under Grouped Round Robin Scheduling, \( N_{t_{\text{comm}}}(\cdot,1) \) is the same for all arms at the end of each session. By this time, all players have also determined the capacity \( C_{\nu} \) for each communication arm candidate \( \nu \) in \(\mathcal{S}^1\) (as per Lemma~\ref*{lemma:capacities in Sp}). 

At this stage, the coordinator selects the communication arm as  
\[
\Carm = \arg\max_{\nu \in \mathcal{S}^1,\; C_\nu < M} \hat{\mu}_{\nu,1},
\]
and then exclusively pulls \( \Carm \) for a duration defined by the function \( \omega(12, t_{\text{comm}}) \).

\begin{definition}[Duration Function \(\omega(a, t)\)]
  \label{def:duration_function}
The duration function 
\(
\omega(\gamma, t) = \left\lceil \frac{\log(\delta/4K^2M)}{\log(1 - \gamma B(N_{t}(\cdot,1)))} \right\rceil
\)
gives the number of consecutive zero rewards needed for all players to detect, with probability at least \( 1 - \delta/4K^2 \), that arm \( \nu \) has switched to the constant-zero distribution. It depends only on the time  \( t \) and a constant \( \gamma \in \mathbb{Z}^+ \) satisfying \( \gamma B(N_{t}(\cdot,1)) \leq \mu_{\nu} \). See Appendix~\ref*{sec:technical_appendices} for detailed discussion.
\end{definition}

Using this duration function, the coordinator exclusively pulls \( \Carm \) for \(\omega(12, t_{\text{comm}}) \) consecutive Grouped Round Robin sessions. This generates a recognizable pattern of zero rewards that serves as a binary signal indicating that \( \Carm \) has been designated as the communication arm. By the construction of \(t_0^1\) and \(t_{\text{comm}}\), we have \( \arm{1} \ge 12\,B(N_{t_{\text{comm}}}(\cdot,1)) > \mu_{\Carm}/9 \) under the good event. This yields 
\(
\omega(12, t_{\text{comm}})  = O( \frac{\log(1/\delta)}{\arm{1}} ).
\)

We assume that there exists at least one arm in \( \mathcal{S}^1 \) whose capacity is strictly less than \( M \). Otherwise, if all top arms, including the best arm \( \sigma_1 \), have capacity greater than or equal to \( M \), then coordination via signaling is not necessary, and the coordinator should instead switch to a standard UCB algorithm.  The pseudocode for the coordinator in Phase 1 is given in Algorithm~\ref{alg:coordinator_protocol_phase1}.

\subsubsection{Listener Protocol (Players \texorpdfstring{$p \neq 1$}{p ≠ 1})}
\label{sec: phase 1 listener}
Similar to the coordinator, each listener \( p \neq 1 \) independently estimates \( \hat{\mu}_{\nu,1}^p(t) \) while playing Grouped Round Robin Scheduling. The set \( \mathcal{S}^p \) is initialized as empty. The first time an arm \( \nu_{p^*} \) satisfies Equation \ref*{eq:inflated radius greater than mu_hat}, player \( p \) adds \( \nu_{p^*}   \) to \( \mathcal{S}^p \), and this time step is denoted as \( t_0^p \). Additionally, all other arms \( \nu' \) satisfying the relaxed condition
\[
\hat{\mu}_{\nu',1}^p(t_0^p) \geq \frac{9}{25} \hat{\mu}_{\nu_{p^*},1}^p(t_0^p)
\]
are also included in the set. The constant \( \frac{9}{25} \) (as opposed to \( \frac{3}{5} \) used in \( \mathcal{S}^1 \)) ensures that \( \mathcal{S}^1 \subseteq \mathcal{S}^p \) under  \(\Egood\) by Lemma~\ref*{lemma:Sone subset Sp}. The player then continues grouped exploration while awaiting a potential signal from the coordinator.

After constructing the set \( \mathcal{S}^p \), each listener \( p \neq 1 \) continues performing repeated full sessions of Grouped Round Robin. During this process, the listener actively checks for communication checkpoints. The listener finds the first checkpoint \( t'_p \) in \(\mathcal{T}_{\text{test}}\) after or at \( t_0^p \). Once \( t'_p \) is reached, the listener enters a listening phase to detect a potential signal from the coordinator. 
A signal is defined as the existence of an arm \( \nu_p \in \mathcal{S}^p \) that incurs \(\omega(12, t'_p) \) unexpected consecutive zero rewards immediately after \( t'_p \), under the condition that the known capacity \( C_{\nu_p} \) is not exceeded\footnote{The capacity of arm \( \nu_p \) is marked as known at time \( t \) if 
\(
\frac{\log\left(N_t(\cdot,2) \cdot MK^2 / \delta\right)}{\hat{\mu}_{\nu_p} - B(N_t(\cdot,2))} \leq N_t(\cdot,2).
\)}.  Note that by \(t_\text{comm}\), listener \(p\) has also learned the capacity \( C_{\nu}\) for each \(\nu \in \mathcal{S}^1\) (as per Lemma~\ref*{lemma:capacities in Sp}).  Since \(\omega(12,t_{\text{comm}})\) only depends on \(t_{\text{comm}}\),  \(\omega(12,t_{\text{comm}}) = \omega(12,t'_{p})\) if \( t'_p \) is the same as \( t_{\text{comm}} \). For example, if \( C_{\nu_p} = 3 \) and the player engages in 2-grouped sessions with \( \nu_p \) resulting in \( \omega(12,t'_p) \) consecutive zero rewards after \( t'_p \), this pattern is interpreted as a signal from the coordinator. If no signal is detected, the player resumes grouped exploration until reaching the next checkpoint. If no signal is received from any arm in \( \mathcal{S}^p \) over three consecutive checkpoints starting from \( t_{0}^p \), the player concludes that the coordinator did not find any arm suitable for communication--implying that the best arm has a capacity of at least \( M \). In this case, the player should switch to playing a standard UCB algorithm over \( \mathcal{S}^p \). Phase 1 for listener and coordinator are formally summarized in Algorithm~\ref*{alg:NonCoordinator_protocol_phase1} and \ref*{alg:coordinator_protocol_phase1} in Appendix~\ref*{sec:algorithm_appendix}.

Figure~\ref{fig:phase1_2} illustrates how the shared event---triggered by inflated confidence thresholds---leads to synchronized signaling and listening phases across players. %
Phase~1 ends when all players have either agreed on the communication arm \( \Carm \) or decided to switch to playing UCB. Once a common \( \Carm \) is established, it can be used in subsequent phases to transmit information. To send a single bit with failure probability at most \( \delta \), the coordinator deliberately breaks the expected capacity of \( \Carm \) for \(  \omega(12,t_{\text{comm}})\) consecutive rounds. This process, called \textit{collision-testing}, induces a sequence of zero rewards that all players can reliably interpret as a binary signal.

\begin{lemma}
  \label{lem:identify_comm_arm}
  If the good event \( \mathcal{E}_{\text{good}} \) holds, then with probability at least \( 1 - \delta/K\)%
  , all non-coordinator players \( p \in \{2, \dots, M\} \) will successfully identify the communication arm \( \Carm \) selected by the coordinator, by executing Algorithm~\ref*{alg:coordinator_protocol_phase1} and Algorithm~\ref*{alg:NonCoordinator_protocol_phase1}.
  \end{lemma}

\begin{figure}[t]
    \centering
    \includegraphics[width=0.46\textwidth]{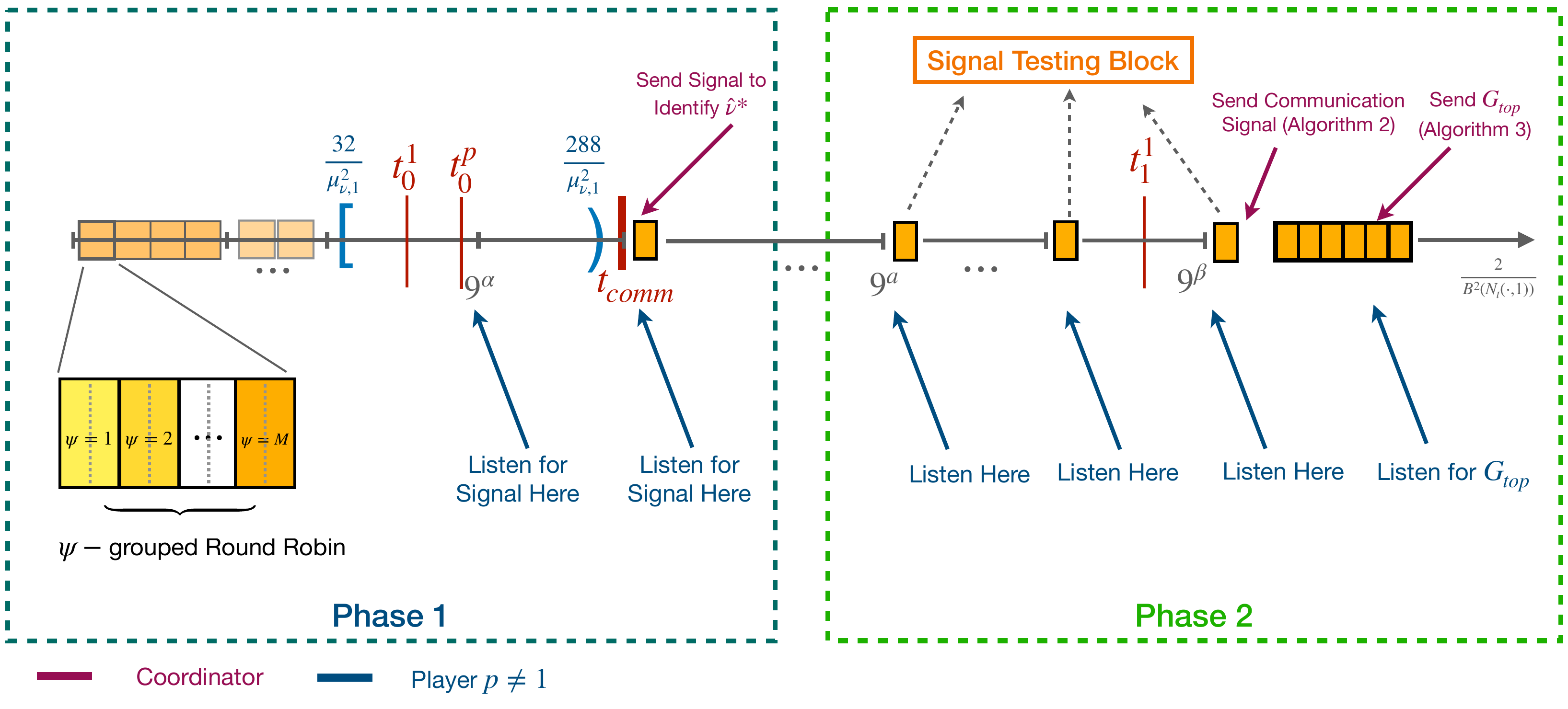}
    \vspace{.1in}
    \caption{
        Illustration of the Coordination Protocol in Phases~1 and~2.
    }
    \label{fig:phase1_2}
\end{figure}

\subsection{Phase 2: Arm Clustering through Connectivity Graphs} 
After identifying the communication arm \( \Carm \) in Phase 1, all players engage in a Simple Round Robin schedule to explore all arms. Starting Phase 2, every time players encounter a checkpoint in \(\mathcal{T}_{\text{test}}\), they insert a \textit{signal testing block} to test for the start of communication. The \textit{signal testing block} is used exclusively for collision testing, and the arm pulls performed during this block are not included in the total pull count tracker \(N(\cdot)\). This process--inserting signal testing blocks into the Simple Round Robin schedule--continues until the coordinator can clearly divide the arms into two distinct clusters. The division is considered clear when the confidence interval of any arm in one cluster is disjoint with the confidence interval of any arm in the other cluster. Once this division is clear, the coordinator uses the communication arm \( \Carm \) to signal the start of communication and transmits the list of arms in the top cluster to all players using \(K\) bits. Figure~\ref*{fig:phase1_2} illustrates the Phase~2 algorithm. The protocol for the coordinator is detailed in Section~\ref{sec:phase2_coordinator}, followed by the protocol for the listeners players in Section~\ref{sec:phase2_listeners}.

\subsubsection{Coordinator Protocol (Player 1)}
\label{sec:phase2_coordinator}
\vspace{-0.1cm}
In Phase 2, the coordinator starts by initializing a \textit{connectivity graph} \( G \) to represent the connections among the arms, where \(\mathcal{C}(G) \) counts the number of connected components within \( G \).

\paragraph{Connectivity Graph \( G \).}
The coordinator maintains a connectivity graph \( G = ( \mathcal{V}, E ) \), where the nodes \( \mathcal{V} \) correspond to the arms. An undirected edge \( (i, j) \in E \) exists if the confidence intervals of arms \( i \) and \( j \) overlap:
\vspace{-0.05cm}
\[
\left| \hat{\mu}_{i,1}^1(t) - \hat{\mu}_{j,1}^1(t) \right| \leq  B(N_t^1(i,1)) + B(N_t^1(j,1) 
\]
The coordinator executes a Simple Round Robin schedule across all arms to update their statistics and the connectivity graph \( G \). 
At each checkpoint \(t\in \mathcal{T}_{\text{test}}\) that occurs after the start of Phase 2, the coordinator inserts a \textit{signal testing block}.
This block involves \(\omega(12,t_{\text{comm}})\) consecutive Simple Round Robin pulls on the arm set \([K] \setminus \Carm\), repeated \(\lceil \frac{M-1}{C_{\Carm}} \rceil\) times. This can also be expressed as repeating \(\text{RR}([K] \setminus \Carm, \omega(12,t_{\text{comm}}))\) for \(\lceil \frac{M-1}{C_{\Carm}} \rceil\) times. Each \(\text{RR}([K] \setminus \Carm, \omega(12,t_{\text{comm}}))\) is designed to allow the transmission of one bit of information to \(C_{\Carm}\) players, repeating it \(\lceil \frac{M-1}{C_{\Carm}} \rceil\) times will enable the coordinator to send a one-bit message to all players. The time step when \(\mathcal{C}(G) > 1 \) is denoted as \( t_1^1 \). The component with the highest estimated mean reward is designated as the top cluster \(G_{\text{top}} = (\mathcal{V}_{\text{top}}, E_{\text{top}})\), while the remaining nodes form the bottom cluster \(G_{\text{bottom}} = (\mathcal{V}_{\text{bottom}}, E_{\text{bottom}})\).

After the event mentioned above happened, the coordinator then searches for the next checkpoint, defined as \(t_{\text{first}} = \min \left\{t: t \in \mathcal{T}_{\text{test}}, t \geq t_1^1\right\}\), where \emph{first} means the first partitioning. After reaching \( t_{\text{first}} \), instead of entering the \textit{signal testing block}, the coordinator uses \( \Carm \) to signal the start of communication through a collision-testing procedure (Algorithm~\ref{alg:collision_testing} in the Appendix) by exclusively pulling 
\(\Carm\) for the same duration as the original block. After signaling the start of communication, the coordinator proceeds to send the top cluster \( G_{\text{top}} \).  In this procedure, a collision indicates a bit value of 1. For each arm \( i \) in the top cluster \( G_{\text{top}} \), the coordinator sends a 1-bit collision to signal its inclusion in the top cluster. This process transmits the encoded message containing the list of arms in the top cluster to all players. The coordinator's protocol in Phase~2 is described in Algorithm~\ref{alg:phase2_coordinator}.

\begin{algorithm}[H]
    \caption{Phase 2: Coordinator}
    \label{alg:phase2_coordinator}
    \begin{algorithmic}[1]
    \State \textbf{Input:} Active players \( \mathcal{P} \), arm set \( \mathcal{A} \), minimum signal testing duration  \(\omega(12,t_{\text{comm}})\)
    \State Initialize connectivity graph \( G  \)
    \While{ \( |\mathcal{C}(G)| \leq 1 \) }
        \If{ \( t \notin \mathcal{T}_{\text{test}} \) }
            \State Perform \textsc{SimpleRoundRobin} (RR) on \( \mathcal{A} \)
            \State Update statistics for all \( i \in \mathcal{A} \) and graph  \( G \)
        \Else
            \State Perform \( \text{RR}([K] \setminus \Carm, \omega(12,t_{\text{comm}})) \) for \(\lceil \frac{M}{C_{\Carm}} \rceil\) times
        \EndIf
    \EndWhile
    \State \( t_{\text{first}} \gets \min\{ t' \in \mathcal{T}_{\text{test}} : t' \geq t \} \)
    \While{ \( t < t_{\text{first}} \) }
        \State Perform \textsc{SimpleRoundRobin} on \( \mathcal{A} \)
        \State Update statistics for all \( i \in \mathcal{A} \) and graph \( G \)
    \EndWhile
    \State CollisionTestingCoordinator (Algorithm~\ref{alg:collision_testing})
    \State \Comment{Signal the Start of Communication}
    \State EncodeArms 
    (Algorithm~\ref{alg:encode_top_cluster})
   \State  \Comment{Send Encoded Message about \(\mathcal{V_\text{top}}\)}
    \State \Return \(\mathcal{V_\text{top}}\)
    \end{algorithmic}
    \end{algorithm}

\subsubsection{Listeners  Protocol (\texorpdfstring{$p \neq 1$}{p ≠ 1})}
\label{sec:phase2_listeners}

Listeners perform Simple Round Robin scheduling on all arms to update their statistics for the arms. After every checkpoint in \(\mathcal{T}_{\text{test}}\), all listeners enter the \textit{signal testing block}. During this block, listeners are grouped into batches of size \(C_{\Carm}\). Any remaining listeners (i.e., if \((M-1) \mod C_{\Carm} \neq 0\)) are assigned to an additional group of size \(C_{\Carm}\). All \(\lceil \frac{M-1}{C_{\Carm}} \rceil\) groups are indexed. Starting with the first group, this group pulls \(\Carm\) exclusively for \(\omega(12,t_{\text{comm}})\) times to test for a potential communication signal from the coordinator (Algorithm~\ref{alg:collision_testing_non_coordinator} in Appendix), while all other players continue exploring with \(\text{RR}([K] \setminus \Carm, \omega(12,t_{\text{comm}})\)). This process is repeated for all groups.

If they detect a signal by the end of a \textit{signal testing block}, which is the block immediately following the checkpoint \(t_{\text{first}}\), they start listening for the message from the coordinator by repeating the \textit{signal testing block} \(K\) times and decode the message from the coordinator     (Appendix~\ref{sec:algorithm_appendix} for details).

\subsection{Phase 3: Capacity Estimation:}
After identifying \( G_{\text{top}} \), the objective is to estimate the capacities \( C_{\nu} \) for all \( \nu \in G_{\text{top}} \), in order to determine which arms require further exploration in the next epoch. Note that Phase~1 guarantees capacity estimation only for arms in the set \( \mathcal{S}^1 \), which may be a strict subset of \( G_{\text{top}} \). Therefore, we initiate a new round of capacity exploration over \( G_{\text{top}} \) using Grouped Round Robin Scheduling (see Algorithm~\ref{alg:capacity_estimation}).

In this phase, the coordinator first computes a lower confidence bound on 
\(
\min_{\nu \in G_{\text{top}}} \mu_\nu,
\)
denoted as \( \gamma \cdot B(N_{t_{\text{first}}}(\cdot,1)) \) for some \( \gamma \in \mathbb{Z}^+ \). The coordinator then encodes and transmits the value of \( \gamma \) using 4 bits (transmitting only its last digit is sufficient). Upon receiving this signal, all listeners perform \( \omega(\gamma, t_{\text{first}}) \) sessions of Grouped Round Robin. By Definition~\ref{def:duration_function}, the duration \( \omega(\gamma, t_{\text{first}}) \) is synchronized across all players. After these sessions, the capacities of all arms in \( G_{\text{top}} \) are known.
\begin{algorithm}[H]
    \caption{Capacity Estimation}
    \label{alg:capacity_estimation}
    \begin{algorithmic}[1]
        \State \textbf{Input:} Set of active players \(\mathcal{P}\), Set of arms \(\mathcal{A}\), Number of sessions \(\omega = \omega(\gamma,t)\)
        \State \textbf{Output:} Minimum capacity \(\psi_a\) for each arm \(a \in \mathcal{A}\)
        
        \State Initialize $\text{capacities}[a] \gets \{1,\ldots,|\mathcal{P}|\}, \forall a \in \mathcal{A}$ \\  \Comment{Initialize all possible capacities}
        \For{$t = 1$ \textbf{to} $\omega$}
           
            \For{$\psi = 1$ \textbf{to} $|\mathcal{P}|$}
                     \State \(\mathcal{P}\) Perform \textsc{$\psi$-GroupedRoundRobin} on \(\mathcal{A}\)

                    \If{observed reward \(r_{a,t}^{(\psi)}\neq 0\) for \(a\in \mathcal{A}\)}
                    \\\color{gray}\Comment{\( r_{a,t}^{(\psi)} \) := the reward at time \( t \) from arm \( a \) when the player is in group size \( \psi \)}\color{black}
                        \State Remove \(\psi\) from $\text{capacities}[a]$
                    \EndIf

            \EndFor
        \EndFor
        \For{each arm \(a \in \mathcal{A}\)}
            \State \(C_a \gets \min(\text{capacities}[a])\) %
        \EndFor
        \State \Return $\{(a, C_a) \mid a \in \mathcal{A}\}$
    \end{algorithmic}
\end{algorithm}

\subsection{Recursion:} 
  
\vspace{-0.1cm}

    At the end of each epoch, players are classified as either \textit{committed}--assigned to a fixed arm and excluded from further exploration--or \textit{active}--continuing to explore arms in \( \mathcal{V}_{\text{top}} \) or \( \mathcal{V}_{\text{bottom}} \). Active players repeat Phase~2 to explore in the next epoch. If the updated \( \mathcal{V}_{\text{top}} \) still contains arms with unknown capacities, Phase~3 is also repeated by the active players. The full recursive process is described in Algorithm~\ref{alg:recursion}.

    Throughout recursion, all players participate in \textit{signal testing blocks} to ensure the reliability of collision-based signaling. This recursive mechanism progressively reduces the coordination problem into smaller sub-instances of MMAB-SAX. Details and pseudocode for the full algorithm can be found in Appendix~\ref*{sec:algorithm_appendix}.
\begin{algorithm}[h]
      \caption{Recursive Allocation Protocol}
      \label{alg:recursion}
      \begin{algorithmic}[1]
        \State \textbf{Input:} Players $[M]$, $\mathcal{V}_\text{top}$, $\mathcal{V}_\text{bottom}$, arm set $[K]$
        \State Initialize \texttt{FIXED} $\gets \emptyset$, $\mathcal{P}_{\text{active}} \gets [M]$
        \While{\texttt{FIXED} does not satisfy Eq.~\ref*{eq:V_arms}}
          \If{$C(\texttt{FIXED} \cup \mathcal{V}_\text{top}) \le M$}
            \State $\texttt{FIXED} \gets \texttt{FIXED} \cup \mathcal{V}_\text{top}$
            \If{$C(\texttt{FIXED}) = M$}
              \State \textbf{Assign} all players to arms in \texttt{FIXED}
              \State \textbf{break}
            \EndIf
            \State Select $\mathcal{P}_{\text{commit}} \subseteq [M]$ of size $C(\texttt{FIXED})$
            \State $\mathcal{P}_{\text{active}} \gets \mathcal{P}_{\text{active}} \setminus \mathcal{P}_{\text{commit}}$
            \State $[K] \gets [K] \setminus \texttt{FIXED}$
            \State $\mathcal{P}_{\text{commit}}$ plays on \texttt{FIXED}
            \State $\mathcal{P}_{\text{active}}$ continue Phase~2 (and Phase~3 if needed) with reduced $[K]$
          \Else
            \State $[K] \gets [K] \setminus \mathcal{V}_\text{bottom}$
            \State $\mathcal{P}_{\text{active}}$ continues Phase~2 on $\mathcal{V}_\text{top} \setminus \texttt{FIXED}$
          \EndIf
        \EndWhile
        \State \textbf{Assign} remaining players to arms in \texttt{FIXED}
      \end{algorithmic}
    \end{algorithm}
\vspace{-0.2cm}
\section{Regret Analysis}
\label{sec:regret_analysis}
\vspace{-0.2cm}
Under the event \(\Egood\), the maximum number of Grouped Round Robin sessions in Phase 1 is 
\(
O(\frac{1}{\arm{1}}^2 \cdot \log(1/\delta)),
\)
since each session contributes at most 
\(
2(MK - M + 1)M \cdot \arm{1}
\)
to the total regret. Consequently, the total regret incurred in Phase 1 is bounded by
\(
O( \frac{(MK - M + 1)M \cdot \log(\log(T)/\delta)}{\arm{1}} ).
\)

In Phase 2, detecting disconnection in the connectivity graph via Simple Round Robin takes at most
  \(
  O(\frac{\log(1/\delta)}{\max_i \gapi{i}^2} 
  )
  \)
  rounds, and each round incurs at most
  \(
  M(K - V)K \cdot \max_i \gapi{i}
  \)
  in regret.
 Communicating \( K +1 \) bits to all players takes
  \(
  O( \log(1/\delta) \cdot \frac{M^2K}{C_{\Carm}}  )
  \)
  regret,  where the term \( C_{\Carm} \) arises from grouping players into groups of size \( C_{\Carm} \), which is necessary to enable collision-based signal detection.
  Since the regret incurred from Simple Round Robin dominates the regret incurred from the \textit{signal testing block}, the total regret in Phase 2 for the first partitioning is:
  \(
  O( \log(\log(T)/\delta)\cdot(\frac{M(K - V)K}{\max_i \gapi{i}} + \frac{M^2K}{C_{\Carm}} )).
  \)

  In Phase 3, 
  we upper bound the duration number of Grouped Round Robin sessions 
  by
  \(
  O( \frac{\log(1/\delta)}{\max_i \gapi{i}} ),
  \) 
  so the total regret incurred in Phase 3 is simplified to
  \(
  O(\frac{(MK - M + 1)M   \arm{1} }{\max_i \gapi{i} }\cdot \log(\log(T)/\delta)).
  \)
Since the recursion process repeats at most \(K-1\) times, the total regret including the regret in Phase 1 is bounded by the following Theorem~\ref{thm:A-CAPELLA_regret}. 
Theorem~\ref{thm:instance_independent_regret} in Appendix  also provides a result with a bound of $O(\sqrt{T})$.

\begin{theorem}[Regret Bound for A-CAPELLA]
  \label{thm:A-CAPELLA_regret}
  Let \( M \) be the number of players, \( K \) the number of arms, and \( V \) the smallest integer such that
  Equation~\ref*{eq:V_arms} holds.
  Let \( \Delta = \min\left( \Delta_{\sigma_{V}, \sigma_{V+1}}, \Delta_{\sigma_{V-1}, \sigma_{V}} \right) \). Then, under the good event \( \mathcal{E}_{\text{good}} \) and setting \( \delta = \frac{1}{T^2MK^2} \), the total regret of the A-CAPELLA algorithm is upper bounded by
  \[
  \mathcal{R}_T = O\left( \log(TMK) \cdot \left[ \frac{K^2 M \cdot \max(M, K - V)}{\Delta} \right] \right).
  \]
  \end{theorem}

\section{NUMERICAL RESULTS }

Since prior communication-based algorithms assume either known capacities or collision/arm-sharing feedback, they cannot operate in our setting: their protocols fail deterministically and yield linear regret. Thus, comparing against these algorithms would not be informative. Instead, we benchmark our method against the \textsc{Selfish} UCB algorithms \citep{besson18}, Randomized \textsc{Selfish} UCB algorithms \citep{Trinh2021}, and adversarially robust EXP3 \citep{auer2002}.

Figure~\ref{fig:regret_vs_horizon} compares \textsc{A-CAPELLA}, \textsc{Selfish} UCB, and EXP3 under a setting with 5 arms, 4 players, uniform arm capacities of 2, and \(\Delta = 0.5\). The results show that \textsc{A-CAPELLA} achieves logarithmic regret. 
In contrast, when \(T > 0.4 \times 10^6\), \textsc{Selfish} UCB algorithms exhibits significantly worse average regret and substantially higher variance across 100 runs, indicating instability. This instability arises because if two players have identical observations at the same time, the construction of \textsc{Selfish} UCB causes them to repeatedly select the same arm, leading to persistent collisions that may continue indefinitely. Furthermore, when \(T > 0.6 \times 10^6\), \textsc{A-CAPELLA} achieves lower average regret compared to EXP3 whose adversarially robust design give a slower convergence.

\begin{figure}[ht]
  \centering
  \includegraphics[width=0.4\textwidth]{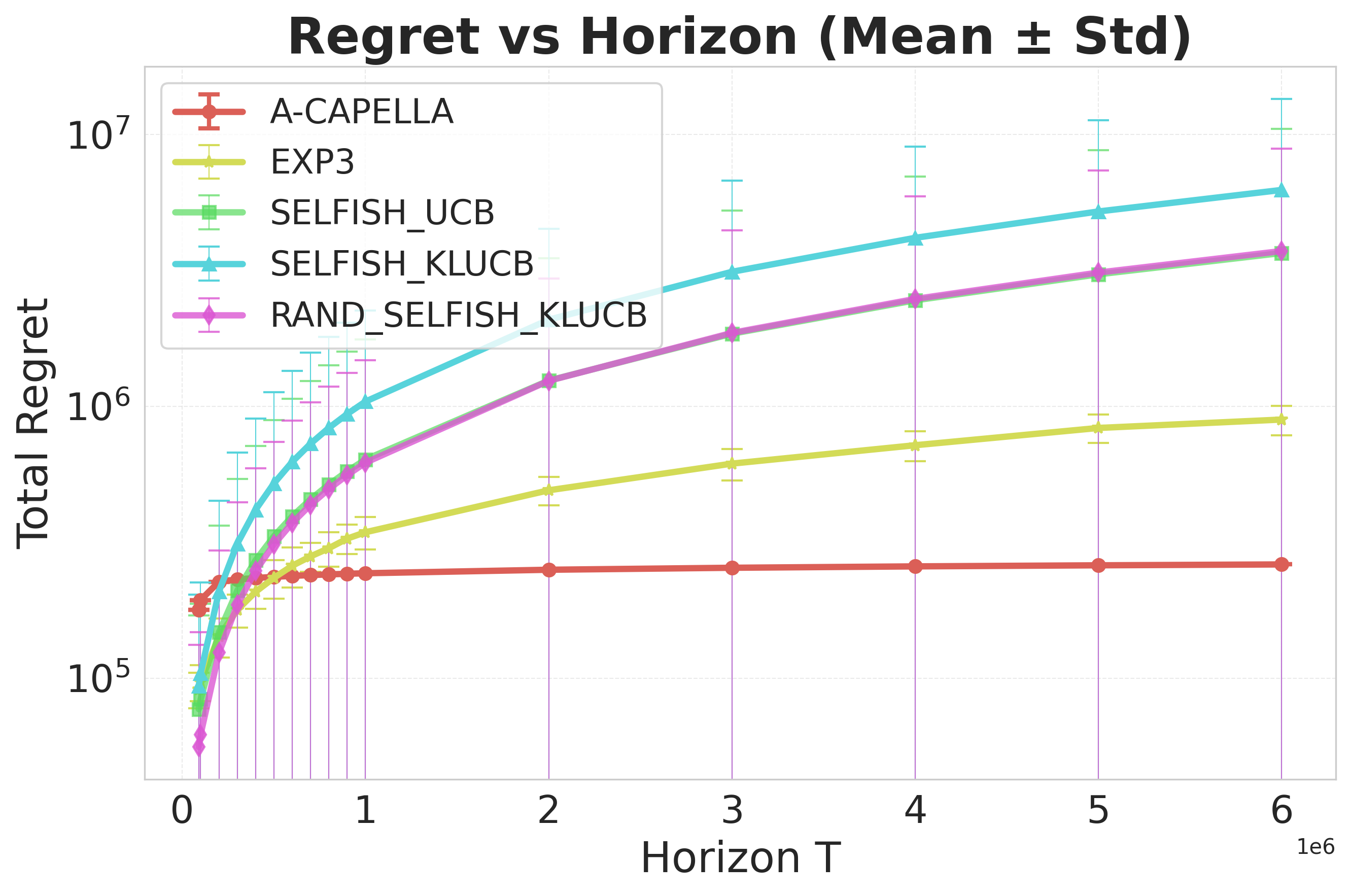}
  \vspace{.1in}
  \caption{
    Comparison of \textsc{A-CAPELLA}, \textsc{Selfish} UCB family, and EXP3. 
    Results are averaged over 100 independent runs; error bars indicate one standard deviation across seeds. 
}

  \label{fig:regret_vs_horizon}
\end{figure}

\section{CONCLUSION}
\label{sec:conclusion}
We present the decentralized algorithm, A-CAPELLA, for the multiplayer multi-armed bandit problem with shareble arms and unkown capacity (MMAB-SAX). Our approach achieves logarithmic regret in the number of rounds, with instance-dependent scaling based on reward gaps. Unlike prior work, we allow arm sharing without direct collision feedback or arm sharing feedback, or assumptions on minimum reward values. Simulations shown in Figure~\ref{fig:regret_vs_horizon} demonstrate sublinear cumulative regret that supports our theoretical guarantees. Overall, this work addresses the no-sensing multiplayer bandit problem under general, unknown capacity constraints.

We acknowledge that our setting lacks a lower bound. Since even the decentralized, no-sensing, unit-capacity case remains open, we contextualize our regret by comparing with prior results. For instance, \citet{pacchiano2023} show  
\(
\mathcal{R}_T \leq \widetilde{\mathcal{O}}\!\left( \frac{M (K - M) K^2 \log T}{\Delta_{\sigma_M, \sigma_{M+1}}} + \mathrm{poly}(K, M, \log T) \right).
\)  
Our regret bound is comparable to those in \citet{pacchiano2023} and \citet{huang2022b}, reflecting the known hardness of decentralized no-sensing MMAB, even with unit capacities.

Prior work has established that some form of communication is essential for achieving instance-dependent guarantees. However, when capacities are unknown and no arm-sharing feedback is available, existing decentralized MMAB algorithms fail, as players cannot exchange information. The added uncertainty of unknown capacities makes coordination especially challenging, since players must first infer capacities before recognizing or interpreting signals.

To overcome this barrier, our work introduces an implicit communication protocol that enables coordination despite the lack of feedback. Building on this, we design an algorithm that, for the first time, achieves a logarithmic, instance-dependent regret guarantee in the fully no-sensing, unknown-capacity setting. We believe our work opens several promising directions: simplifying the algorithmic complexity, improving regret scaling in large-scale settings, and relaxing assumptions such as players knowing $M$ and their own indices. We leave these challenges for future research.

 \newpage
\bibliography{reference}

\clearpage
\appendix
\thispagestyle{empty}

\onecolumn

\aistatstitle{
Supplementary Materials}
\section{TECHNICAL APPENDICES}
\label{sec:technical_appendices}
\subsection{Theorems and proofs}
In this section, we present the detailed proof of theorems and lemmas mentioned in the main paper. 
\begin{lemma}
  \label{lem:confInt_bound}
  Let \( t_0 \) be the first time step at which the following condition holds for \( h > 2 \):
  \begin{equation}
      \hat{\mu}_{\nu,1}(t) - h \cdot B(N_t(\nu,1)) > 0.
      \label{eq:special-round-condition}
  \end{equation}
  Then, under the good event \( \mathcal{E} \), the following inequalities are satisfied:
  \begin{align}
      \frac{\mu_{\nu,1}}{2(h+1)} &< B(N_{t_0}(\nu,1)) \leq \frac{\mu_{\nu,1}}{h - 1}, \label{eq:b-bounds} \\
      \frac{2(h - 1)^2}{\mu_{\nu,1}^2} &\leq \frac{N_{t_0}(\nu,1)}{g(N_{t_0}(\nu,1))} < \frac{8(h + 1)^2}{\mu_{\nu,1}^2}. \label{eq:n-over-g}
  \end{align}
\end{lemma}

\begin{proof}
Under event \( \mathcal{E} \), we have:
\(
\hat{\mu}_{\nu,1}(t) \geq \mu_{\nu,1} - B(N_t(\nu,1)).
\)
Suppose for contradiction   that for \( t <t_0 \), \( B(N_t(\nu,1)) \leq \frac{\mu_{\nu,1}}{h + 1} \). Then:
\[
\hat{\mu}_{\nu,1}(t) \geq \mu_{\nu,1} - \frac{\mu_{\nu,1}}{h + 1} = \frac{h}{h + 1} \mu_{\nu,1} \geq h \cdot B(N_t(\nu,1)),
\]
which contradicts the assumption that \( \hat{\mu}_{\nu,1}(t) - h \cdot B(N_t(\nu,1)) \leq 0 \) before \(t_0\). 
Thus, by contraposition, we conclude:
\[
B(N_t(\nu,1)) > \frac{\mu_{\nu,1}}{h + 1}
\] for \(t <t_0\). 
By continuity of \(B(\cdot)\), we also have that if \( \delta < \frac{1}{2} \), then it is easy to see that:
\[
B(N_{t-1}(\nu,1)) < 2 B(N_{t}(\nu,1)) \quad \Rightarrow \quad B(N_{t_0}(\nu,1)) >\frac{1}{2}B(N_{t_0 -1}(\nu,1)) > \frac{\mu_{\nu,1}}{2(h + 1)}.
\]

Next, for the upper bound, from the definition of \( \mathcal{E} \), we also have:
\[
\hat{\mu}_{\nu,1}(t) \leq \mu_{\nu,1} + B(N_t(\nu,1)).
\]
Rewriting the condition \( \hat{\mu}_{\nu,1}(t_0) \geq h \cdot B(N_{t_0}(\nu,1)) \), we subtract \( B(N_{t_0}(\nu,1)) \) from the equation above from both sides:
\[
(h - 1) B(N_{t_0}(\nu,1)) \leq \hat{\mu}_{\nu,1}(t_0) - B(N_{t_0}(\nu,1))  \leq \mu_{\nu,1}.
\]
This yields:
\[
B(N_{t_0}(\nu,1)) \leq \frac{\mu_{\nu,1}}{h - 1}.
\]

Combining both bounds, we obtain:
\[
\frac{\mu_{\nu,1}}{2(h + 1)} < B(N_{t_0}(\nu,1)) \leq \frac{\mu_{\nu,1}}{h - 1}.
\]

To convert this into a bound on the normalized sample count, recall:
\[
B(N) = \sqrt{\frac{g(N)}{N}} \quad \Rightarrow \quad \frac{N}{g(N)} = \frac{1}{B(N)^2}.
\]
Therefore,
\[
\left( \frac{\mu_{\nu,1}}{h - 1} \right)^{-2} \leq \frac{N_{t_0}(\nu,1)}{g(N_{t_0}(\nu,1))} < \left( \frac{\mu_{\nu,1}}{2(h + 1)} \right)^{-2},
\]
which simplifies to:
\[
\frac{2(h - 1)^2}{\mu_{\nu,1}^2} \leq \frac{N_{t_0}(\nu,1)}{g(N_{t_0}(\nu,1))} < \frac{8(h + 1)^2}{\mu_{\nu,1}^2}.
\]

When \( h = 5 \), this becomes:
\[
\frac{32}{\mu_{\nu,1}^2} \leq \frac{N_{t_0}(\nu,1)}{g(N_{t_0}(\nu,1))} < \frac{288}{\mu_{\nu,1}^2}.
\]

There must exist a unique \(9^\alpha \in [\frac{32}{\mu_{\nu,1}^2} , \frac{288}{\mu_{\nu,1}^2} )\) for \(\alpha \in \mathbb{N}^+\) according to Lemma 9 in \cite{pacchiano2023}.
\end{proof}

\begin{lemma}
  \label{lemma:Sone subset Sp}
  If \(\Egood\) holds, with probability at least \(1 - \delta/MK\), the optimal arm \(\sigma_1\) lies in the set \(\Sone\) at time \(t_0^1\), and moreover \(\Sone \subseteq \Sp\) for all \(p \in K\) with \(p \neq 1\), at time \(t_0^p\).
  \end{lemma}
  
  \begin{proof}
  By definition, \(\sigma_1\) is the optimal arm, so
  \[
    \arm{1} \;=\; \max_{i \in K}\,\mu_i.
  \]
  Let \(\Carm\) be the first arm such that \eqref{eq:inflated radius greater than mu_hat} holds for the first time at time \(t_0^1\), and the following holds at \(t_0^1\) by definition
  \[
    B\bigl(N_{t_0^1}^p(\Carm,1)\bigr)
    \;\le\;
    \frac{\mu_{\Carm}}{4}
    \;\le\;
    \frac{\arm{1}}{4},
  \]
  where \(B(\cdot)\) is the confidence bound radius and \(N_{t_0^1}(\cdot,1)\) denotes the relevant sample count up to time \(t_0^1\).
  
  Since \(\sigma_1\) is truly optimal, it follows that
  \[
    \hat{\mu}_{\sigma_1}(t_0)
    \;\;\ge\;
    \arm{1}
      \;-\;
    B\bigl(N_{t_0^1}(\sigma_1,1)\bigr)
    \;\;\ge\;
    \arm{1}
      \;-\;
    \frac{\arm{1}}{4}
    \;=\;
    \frac{3}{4}\,\arm{1}.
  \]
  By the choice of \(t_0^1\), and any arm \(\nu\in\Sone\) satisfies \(\hat{\mu}_\nu \geq \frac35\hat{\mu}_{\Carm}(t_0^p)\), we also have
  \[
    \frac35\,\hat{\mu}_{\Carm}(t_0^1)
    \;\;\le\;
    \frac35\bigl(\mu_{\Carm} + B(N_{t_0^1}(\Carm,1))\bigr)
    \;\;\le\;
    \frac35\,\arm{1}
      + \frac{3}
      {20}\,\arm{1}
    \;=\;
    \frac34\,\arm{1}
    \;\;\leq\;
    \hat{\mu}_{\sigma_1}(t_0^1).
  \]
  Hence,
  \[
    \hat{\mu}_{\sigma_1}(t_0^1)
    \;\;\geq\;
    \frac35\,\hat{\mu}_{\Carm}(t_0^1),
  \]
  which shows \(\sigma_1 \in \Sone\).
  
  \medskip
  
  Next we prove \(\Sone \subseteq \Sp\) for each player \(p \neq 1\). Fix any player \(p\), and let \(j^*\) be the arm for which \eqref{eq:inflated radius greater than mu_hat} holds first for such a player \(p\). Then by the definition of \(\Sp\), any arm \(\nu\in\Sp\) satisfies \(\hat{\mu}_\nu \geq \frac{9}{25}\hat{\mu}_{j^*}^p(t_0^p)\). Furthemore, 
  \[
    \frac{9}{25}\,\hat{\mu}_{j^*}^p(t_0^p)
    \;\;\le\;
    \frac{9}{25}
    \bigl(\mu_{j^*} + B(N_{t_0^p}(j^*,1))\bigr)
    \;\;\le\;
    \frac{9}{25}\cdot \frac{5}{4}\arm{1}
    \;=\;
    \frac{9}{20}\,\arm{1}.
  \]
  On the other hand, for every \(i \in \Sone\), \(\hat{\mu}_{i}(t_0^1)\ge
    \frac35\,\hat{\mu}_{\Carm}(t_0^1)\), which is at least
  \begin{equation}
  \frac{3}{5}
  \Bigl(\mu_{\Carm} - B\bigl(N_{t_0^1}(\sigma_1,1)\bigr)\Bigr)
  \;\;\ge\;
  \frac{3}{5}
  \Bigl(\arm{1} - B\bigl(N_{t_0^1}(\sigma_1,1)\bigr)\Bigr)
  \;\;\ge\;
  \frac{3}{5} \cdot \frac{3}{4} \arm{1}
  \;=\;
  \frac{9}{20}\,\arm{1}
\end{equation}

  Combining these, $\forall i \in \Sone$
  \[
    \hat{\mu}_{i}(t_0^1)
    \;\;\ge\;
    \frac{9}{20}\,\arm{1}
    \;\;\ge\;
    \frac{9}{25}\,\hat{\mu}_{j^*}^p(t_0^p),
  \]
  which implies $i \in \Sp$. Thus $\Sone \subseteq \Sp$.
  \end{proof}

\begin{lemma}
\label{lemma:capacities in Sp}
Let \(\delta\in (0,1)\). With probability at least \( 1 - \delta/K \), at time \( t_0^1 \), each player \( p \in [M] \) has correctly identified the capacities of all arms \( i \in \mathcal{S}^1 \).
\end{lemma}

    \begin{proof}
     Let \( t_0^p \) denote the first time step at which player \( p \) constructs the set \( \mathcal{S}^p \). By this time, for each arm \( j \in \mathcal{S}^p \) and each group size \( \psi \in [M] \), the player has collected at least \( N_{t_0^p}(j,\psi) \) samples. In particular, due to the scheduling structure, we have \( N_{t_0^p}(j,\psi) \geq \frac{1}{2} N_{t_0^p}(j,1) \) for any \( \psi > 1 \). This relation follows from the grouped Round Robin protocol.
    
    From Lemma~\ref{lem:confInt_bound}, we know that the solo-pull sample count satisfies the inequality
    \[
    \frac{N_{t_0^p}(j,1)}{g(N_{t_0^p}(j,1))} \geq \frac{32}{\mu_{j^p}^2},
    \]
    where \( \mu_{j^p} \) denotes the mean reward of the best arm at time \(t_0^p\) for player \(p
    \). This implies
    \[
    N_{t_0^p}(j,1) \geq \frac{32}{\mu_{j^p}^2} \cdot \log\left(4 N_{t_0^p}^2(j,1) MK / \delta\right).
    \]
    Combining this with the earlier inequality, we obtain
    \begin{equation}
     \label{eq:sample until sp}
    N_{t_0^p}(j,\psi) \geq \frac{16}{\mu_{j^p}^2} \cdot \log\left(4 N_{t_0^p}^2(j,1) MK / \delta\right) \geq \frac{16}{\arm{1}^2} \cdot \log\left(4 N_{t_0^p}^2(j,1) MK / \delta\right).
    \end{equation}
    
    Next, we consider the number of samples needed to determine the capacity of arm \( j \). To distinguish whether an arm has capacity at least \( \psi \), we can treat the problem as a binary hypothesis test: \( H_1 \colon \mu_{j,\psi} = 0 \) versus \( H_0 \colon \mu_{j,\psi} > 0 \). This test can be conducted with \( \frac{\log(\delta')}{\log(1 - \mu_j)}\) samples to achieve \( 1 - \delta' \) confidence (Lemma~\ref{lem:zero_testing_samples}). Using the inequality \( \log(1 - x) \leq -x \) for \( x \in (0,1) \) (assuming log is base 2), we find only
    \begin{equation}
    \frac{\log( \delta')}{\log(1 - \mu_j)} \leq \frac{\log(1/\delta')}{\mu_{j}}.
    \label{eq:lower_bound_sp}
    \end{equation} consecutive zeros are required to determine the capacity of an arm with high probability.
    
    All arms \(j\) in \( \mathcal{S}^p \) satisfies \(\hat{\mu}_j\geq \frac{9\hat{\mu}_{j^p}}{25}\) (or \(\hat{\mu}_j\geq \frac{3\hat{\mu}_{j^1}}{5}\) for the coordinator), and the empirical mean is at least%

\begin{equation}
 \frac{9}{25} \hat{\mu}_{\sigma_1}
 \;\;\ge\;
  \frac{9}{25}
  \Bigl(\arm{1} - B\bigl(N_{t_0^p}(\sigma_1,1)\bigr)\Bigr),
\end{equation}
where \(B\bigl(N_{t_0^p}(\sigma_1,1)\bigr) \leq \arm{1}/4\). 
We can further bound the true mean of arms \(j\) in \( \mathcal{S}^p \) by
\(\hat{\mu}_j - B\bigl(N_{t_0^p}(\sigma_1,1)\bigr)\) again, and this give us a lower bound for \(\mu_j \) in terms of \(\Omega(\arm{1})\).  Similarly, for \(i\) in \( \mathcal{S}^1 \) satisfies \(\hat{\mu}_i\geq \frac{3\hat{\mu}_{i^1}}{5}\), and 

\begin{equation}
\mu_{i}
  \;\;\ge\;
  \frac{3}{5}
  \Bigl(\arm{1} - B\bigl(N_{t_0^1}(\sigma_1,1)\bigr)\Bigr)-B\bigl(N_{t_0^1}(\sigma_1,1)\bigr)
  \;\;\ge\;
  \frac{1}{5}\arm{1} ,
\end{equation}

After substituting this lower bound into Eq.~\ref{eq:lower_bound_sp}, we obtain that the number of samples required across players to identify the capacities of arms in \(\Sone\) is smaller than the number already observed up to time \( t_0^1 \), as stated in Eq.~\ref{eq:sample until sp}. Therefore, by setting \(\delta' = \delta/N_{t_0^1}(\cdot,2)MK^2\)  with probability at least \(1- \delta/K\) all players can identify the capacity of all arms in \(\Sone\) by time \(t_0^1\). Furthermore, with simple algebra it is easy to see that each player has known the capacity for arms in \(\Sp\) by \(t_\text{comm}\).

    \end{proof}

    \begin{lemma}
      Let \(\delta\in (0,1)\). For any arm $i$, the minimum number of samples needed to transmit a 1-bit message using collision-testing with a confidence level of $1-\delta$ is given by $\frac{\log(\delta)}{\log(1-\mu_i)}$.
      \label{lem:zero_testing_samples}
  \end{lemma}

  \begin{proof}
  We can formulate the problem as a \textbf{hypothesis testing} problem:
  \[ H_0: \mu_i>0 \text{ vs } H_1: \mu_i=0.\]
  
  Denote by \(\mathbb{P}_{H0}\) the underlying probability measure if \(H_0\) is true (sim. \(\mathbb{E}_0\)), and, vice-versa, by \(\mathbb{P}_{H1}\) the underlying probability measure if \(H_1\) is true (sim. \(\mathbb{E}_1\)). Furthermore, denote by \(n\)  the number of trials needed to reject \(H_0\)  with confidence at-least \(1-\delta\).
  
  \paragraph{Probability of getting 0 under \(H_0\).} Define \(p = \mathbb{P}_{H_0}(X = 0)\). $$\mathbb{E}(X) = 0*p + \int_{(0,1]}x dP(x).$$ Let \(\tilde{X} = \frac{1}{1-p}\int_{(0,1]}x dP(x) \), which is the averaged of X under the region \(X>0\), and \(\tilde{X}\leq 1\) because \(X\in [0,1]\). Then $$\mathbb{E}(X) = (1-p)\tilde{X}\leq (1-p).$$
  When \(\mathbb{E}(X) = \mu_i\), \(\mu_i\leq(1-p)\), which implies \(\mathbb{P}_{H0}(X = 0) = p\leq 1-\mu_i\)

  \paragraph{Trials under \(H_0\).} Assume \(H_0\) is true. To determine \(\delta\)  we  can  consider the probability of observing \(n\)  zeros under \(H_0\) and set it equal to our confidence level
  \[
  \mathbb{P}_{H_0}(X_1=0,\dots, X_n=0|\mu_i) \leq  (1-\mu_i)^n \leq  \delta.
  \]
  Hence, we reject \(H_0\) if we observe \( n \geq \frac{\log (1/\delta)}{\mu_i}  > \frac{\log (\delta)}{\log(1-\mu_i)} \) zeros consecutively when testing for capacity. We get the first inequality because \(\log(1-x)\leq -x\) for \(x\in (0,1)\) if log has base less than \(e\). This guarantees that the type-I error is at-most \(\delta\).

  \paragraph{Trials under \(H_1\).} Assume \(H_1\) is true.
\(
  \mathbb{P}_{H_1}(\exists i \, X_i \neq 0 \mid \mu_i) = 0.
  \)
This is true by the definition of \(H_1\) under collision. Thus, the type-II error is 0.

  \end{proof} 
  With Lemma~\ref{lem:zero_testing_samples}, we can show proof of Lemma~\ref{lem:identify_comm_arm} 

  \begin{proof}[Proof of Lemma~\ref{lem:identify_comm_arm}] 
    Fix any non-coordinator player \(p\in\{2,\dots,M\}\).
    After player \(p\) has formed the set \(\mathcal S^{p}\) at time \(t_0^{p}\),
    they watches the next three checkpoints \(t'\in\mathcal{T}_{\text{test}}\) with
    \(t'\ge t_0^{p}\).
    At each such checkpoint the protocol observes exactly
    \(\omega\!\bigl(12,t'\bigr)\) complete Grouped Round-Robin sessions.
    By Definition~\ref{def:duration_function} this duration satisfies
    \[
    \omega\!\bigl(12,t'\bigr)
    \;\ge\;
    \frac{\log\!\bigl(\delta/4K^{2}M\bigr)}{\log\!\bigl(1-\mu_{i}\bigr)},
    \qquad
    \text{for every arm }i\in\mathcal S^{p},
    \]
    because \( 12\,B\!\bigl(N_{t'}(\,\cdot\!,1)\bigr)\le\mu_{i}\) provides a
    uniform lower bound on the mean reward throughout the phase.
    
    \medskip
    Consequently, after each checkpoint player \(p\) collects
    at least the sample size required by
    Lemma~\ref{lem:zero_testing_samples}
    to distinguish ``capacity maintained'' from ``capacity broken unexpectedly''
    on every candidate arm \(i\in\mathcal S^{p}\) with
    type-I error probability at most \(\delta/(4K^{2}M)\).
    
    \medskip
    The collision-testing block is executed at most three times.
    Applying the union bound over the three checkpoints, the
    \(K\) arms, and the \(M\) players yields a
    total type-I error probability no greater than \(\delta/K\).
    Hence, with probability at least \(1-\delta/K\),
    every non-coordinator player correctly detects the arm on which the
    coordinator generates the collision pattern.
    \end{proof}

\begin{lemma}[Regret of a Single Simple Round-Robin Round]
  \label{lemma:simple round robin regret}
  Let \(M\) players interact with \(K\) arms whose means are ordered
  \(\mu_{\sigma_1}> \mu_{\sigma_2}>\dots>\mu_{\sigma_K}\).
  Define \(V\) as the smallest integer satisfying
  \(
  \sum_{i=1}^{V-1} C_{\sigma_i}<M\le\sum_{i=1}^{V} C_{\sigma_i},
  \)
  Denote one full Simple Round-Robin cycle by
  \(R_{\mathrm{sr}}\).
  Then
  \[
  R_{\mathrm{sr}}
  \;=\;
  K\!\sum_{i=1}^{V}\mu_{\sigma_i}\tilde C_{\sigma_i}
  \;-\;
  M\!\sum_{j=1}^{K}\mu_{\sigma_j}
  \;\;\le\;\;
  M\,(K-V)\,(\arm{1}-\arm{K})
  \;\;\le\;\;
  M\,(K-V)\,K\,\max_{i\in[K-1]}\Delta_{\sigma_i,\sigma_{i+1}},
  \]
  where \(\Delta_{\sigma_i,\sigma_{i+1}}=\mu_{\sigma_i}-\mu_{\sigma_{i+1}}\).
  \end{lemma}

\begin{proof}
During one round of simple Round-Robin, all \( M \) players collectively pull arms \( MK \) times. The expected reward can be written as
\(
M \sum_{j=1}^{K} \mu_{\sigma_j},
\)
where \( \mu_{\sigma_j} \) is the expected reward of arm \( \sigma_j \). The optimal reward is 
\(
K \sum_{i=1}^{V} \mu_{\sigma_i} \tilde{C}_{\sigma_i},
\)
where \( \tilde{C}_{\sigma_i} \) represents the optimal number of players allocated to each arm, ensuring that exactly \( \tilde{C}_{\sigma_j} \) players pull arm \( \sigma_j \). The value of \( V \) is determined as the smallest integer satisfying \(
\sum_{i=1}^{V-1} C_{\sigma_i} < M \leq \sum_{i=1}^{V} C_{\sigma_i},
\)
where \( C_{\sigma_i} \) is the maximum number of players that arm \( \sigma_i \) can sustain.

\[
\tilde{C}_{\sigma_i} =
\begin{cases} 
C_{\sigma_i}, & \text{if } i \leq V-1 \\ 
M - \sum_{j=1}^{V-1} C_j, & \text{if } i = V.
\end{cases}
\]
Then, we have
\(
\sum_{i=1}^{V} \tilde{C}_{\sigma_i} = M
\).
The regret collected during one round of simple Round-Robin is given by:
\[
R_{sr} = K \sum_{i=1}^{V} \mu_{\sigma_i} \tilde{C}_{\sigma_i} - M \sum_{i=1}^{K} \mu_{\sigma_j}.
\]
Decomposing the regret into the sum of the difference between the reward collected from optimal allocation and reward collected from round-robin allocation on each arm, we have:
\begin{align*}
R_{sr} &= \sum_{i=1}^{V} \left[ (K \tilde{C}_{\sigma_i}- M )\mu_{\sigma_i} \right]
- M\sum_{i=V+1}^{K} \mu_{\sigma_j}\\
&\leq \sum_{i=1}^{V} \left[ (K \tilde{C}_{\sigma_i}- M )\mu_{\sigma_1} \right] - M\sum_{i=V+1}^{K} \mu_{\sigma_k}\\
&= M(K -V )\mu_{\sigma_1} - M(K-V) \mu_{\sigma_k}\\
&= M(K -V )(\mu_{\sigma_1} - \mu_{\sigma_k})\\
&\leq M(K -V )K\max_i\Delta_{\sigma_i,\sigma_{i+1}}.
\end{align*}

\end{proof}

\begin{lemma}[Regret of a Complete \(\psi\)-Grouped Round-Robin Session]
  \label{lemma:grouped round robin regret}
  Let \(M\) players interact with \(K\) arms under one complete session of \(\psi\)-Grouped Round-Robin for each \(\psi \in [M]\). Let \(\sigma_1, \dots, \sigma_K\) index arms in decreasing order of mean reward, and define \(V\) and \(\tilde{C}_{\sigma_i}\) as in Lemma~\ref{lemma:simple round robin regret}. Then the total regret incurred across all \(\psi\)-Grouped Round-Robin sessions satisfies:
  \[
  R_{r_\psi} \leq 2(MK - M+1)M\arm{1},
  \]
  \end{lemma}
  
  \begin{proof}
  We analyze a complete session of \( \psi \)-Grouped Round-Robin (\( r_\psi \)) for all \( \psi \in [M] \), where each group consists of \( \psi \) players.
  
  For each \( r_\psi \) round, there are \( \left\lceil \frac{M}{\psi} \right\rceil \) groups. The total expected reward collected on each arm during one round is:
  \[
  X_{\sigma_i} = 2\mu_{\sigma_i}\left[\left\lfloor \frac{M}{\psi} \right\rfloor\cdot \psi\cdot \mathbb{I}(\psi \leq C_{\sigma_i}) + \rho\cdot \mathbb{I}(\rho \leq C_{\sigma_i})\right], \quad \text{where } \rho = M \bmod \psi.
  \]
  
  The regret for a complete session of \( \psi \)-Grouped Round-Robin is given by:
  \begin{align*}
  R_{r_\psi} &= 2\sum_{\psi=1}^{M} \left[ K \sum_{i=1}^{V} \mu_{\sigma_i} \tilde{C}_{\sigma_i} - \sum_{i=1}^{K}\mu_{\sigma_i}\left[\left\lfloor \frac{M}{\psi} \right\rfloor\cdot \psi\cdot \mathbb{I}(\psi \leq C_{\sigma_i}) + \rho\cdot \mathbb{I}(\rho \leq C_{\sigma_i})\right] \right]\\
  &\leq 2\sum_{\psi=1}^{M} \left[ K \sum_{i=1}^{V} \mu_{\sigma_i} \tilde{C}_{\sigma_i} - \sum_{i=1}^{K}\mu_{\sigma_i}\left\lfloor \frac{M}{\psi} \right\rfloor\cdot \psi\cdot \mathbb{I}(\psi \leq C_{\sigma_i})\right]\\
  &\leq 2MK \sum_{i=1}^{V} \mu_{\sigma_i} \tilde{C}_{\sigma_i} - 2\sum_{\psi=1}^{M}\sum_{i=1}^{K}\mu_{\sigma_i}(M-1)\mathbb{I}(\psi \leq C_{\sigma_i})\\
  &= 2MK \sum_{i=1}^{V} \mu_{\sigma_i} \tilde{C}_{\sigma_i} - 2(M-1)\sum_{i=1}^{K}\mu_{\sigma_i}\sum_{\psi=1}^{M}\mathbb{I}(\psi \leq C_{\sigma_i})\\
  &= 2MK \sum_{i=1}^{V} \mu_{\sigma_i} \tilde{C}_{\sigma_i} - 2(M-1)\sum_{i=1}^{K}\mu_{\sigma_i}C_{\sigma_i}\\
  &= 2\sum_{i=1}^{V} \left[MK\mu_{\sigma_i} \tilde{C}_{\sigma_i} - (M-1)\mu_{\sigma_i}C_{\sigma_i}\right] -2(M-1)\sum_{i=V+1}^{K}\mu_{\sigma_i}C_{\sigma_i}\\
  &\leq 2\sum_{i=1}^{V} \left[(MK - M+1)\tilde{C}_{\sigma_i}\mu_{\sigma_i} \right] -2(M-1)\sum_{i=V+1}^{K}\mu_{\sigma_i}C_{\sigma_i}
  \end{align*}
  
  To bound the first term:
  \begin{align*}
      & 2\sum_{i=1}^{V} \left[(MK - M+1)\tilde{C}_{\sigma_i}\mu_{\sigma_i} \right]\\
      &\leq 2(MK - M+1)\mu_{\sigma_1} \sum_{i=1}^{V} \tilde{C}_{\sigma_i} \\
      &= 2(MK - M+1)M\mu_{\sigma_1}
  \end{align*}
  \end{proof}

\begin{proof}[Proof of Theorem~\ref{thm:A-CAPELLA_regret}] We decompose the total regret into three components, corresponding to the phases outlined in the algorithm. We first analyze the regret incurred in \textbf{Phase~1} and during the \textbf{first partitioning} performed in \textbf{Phases~2 and~3}, and then analyze the regret accumulated during the \textbf{recursive process}, which repeatedly applies the same partitioning procedure.

  Under the event \( \Egood \), the maximum number of Grouped Round-Robin sessions in Phase~1 is
  \(
  O\left(\frac{1}{\arm{1}^2} \cdot \log\left(\frac{1}{\delta}\right)\right),
  \)
  since each session contributes at most
  \(
  2(MK - M + 1)M \cdot \arm{1}
  \)
  to the total regret, as established in Lemma~\ref{lemma:grouped round robin regret}. Consequently, the total regret incurred in Phase~1 is bounded by
  \[
  O\left( \frac{(MK - M + 1)M \cdot \log\left(\log(T)/\delta\right)}{\arm{1}} \right),
  \]
  where the additional \( \log(T) \) factor arises because the algorithm only checks for signaling at predefined checkpoints.
  
  In Phase~2, detecting a disconnection in the connectivity graph via Simple Round-Robin takes at most
  \[
  O\left(\frac{\log(1/\delta)}{\max_i \gapi{i}^2} \right)
  \]
  rounds, and each round incurs at most
  \(
  M(K - V)K \cdot \max_i \gapi{i}
  \)
  in regret by Lemma~\ref{lemma:simple round robin regret}. Thus, the total regret from Simple Round-Robin incurred in this phase is bounded by
  \[
  O\left( \frac{M(K - V)K \cdot \log(1/\delta)}{\max_i \gapi{i}} \right).
  \]
  
  The regret from collision testing and communication arises from transmitting an `ON' signal and \( K \) bits to all players about the arms in \( \mathcal{V}_{\text{top}} \). To transmit a single bit to one player with failure probability less than \( \delta \), we require
  \(
  O\left( \frac{\log(1/\delta)}{\arm{1}} \right)
  \)
  collision-testing rounds. Since each signal must reach all \( M \) players, and each group can contain at most \( C_{\Carm} \) players, we require approximately \( \lceil \frac{M}{C_{\Carm}} \rceil\) collision-testing blocks per bit. Each round in the collision-testing block may incur up to \( O(M \cdot \arm{1}) \) regret, since all \( M \) players may simultaneously incur regret \( \arm{1} \). Therefore, the total regret from communication is
  \(
  O\left( \log(1/\delta) \cdot \frac{M^2K}{C_{\Carm}} \right).
  \)
  
  Since the regret from Simple Round-Robin dominates the regret from the \emph{signal testing block} that precedes communication, the total regret in Phase~2 for the first partitioning is bounded by:
  \[
  O\left( \log\left(\frac{\log T}{\delta} \right) \cdot \left( \frac{M(K - V)K}{\max_i \gapi{i}} + \frac{M^2K}{C_{\Carm}} \right) \right).
  \]

  In Phase~3, we can ignore the regret incurred from the four collision-testing blocks used to communicate \( \gamma \), and group that regret with the collision-testing regret already accounted for in Phase~2. We now focus on bounding the number of Grouped Round-Robin sessions required for capacity estimation and its regret.

  To test for capacity we need at least \( \lceil \frac{\log(1/\delta)}{\min_{i \in \mathcal{V}_{\text{top}}} \mu_{i}} \rceil \) Grouped Round-Robin sessions. We upper bound the duration (i.e., the number of Grouped Round-Robin sessions) by
  \(
  O\left( \frac{\log(1/\delta)}{\max_i \gapi{i}} \right),
  \)
  so the total regret incurred by Grouped Round-Robin in Phase~3 is simplified to
  \[
  O\left( \frac{(MK - M + 1)M \cdot \arm{1}}{\max_i \gapi{i}} \cdot \log(\log(T)/\delta) \right).
  \]
  
  Combining the regrets from Phase~1, Phase~2 (including communication), and Phase~3 for the first partitioning, we obtain the total regret bound:
  \[
  O\left( \log\left( \frac{\log(T)}{\delta} \right) \cdot \left[ 
  \frac{(MK - M + 1)M}{\arm{1}} 
  + \frac{M(K - V)K + (MK - M + 1)M \cdot \arm{1}}{\max_i \gapi{i}} 
  + \frac{M^2K}{C_{\Carm}} 
  \right] \right).
  \]

  Since the recursion process repeats at most \( K - 1 \) times, the total regret is bounded by:
  \[
  O\left( \log\left( \frac{\log(T)}{\delta} \right) \cdot \left[
  \frac{(MK - M + 1)M}{\arm{1}} 
  + \frac{M(K - V)K^2 + (MK - M + 1)MK \cdot \arm{1}}{\Delta} 
  + \frac{M^2K}{C_{\Carm}} 
  \right] \right),
  \]
  where \( \Delta = \min\left( \gapi{V}, \Delta_{\sigma_{V-1}, \sigma_{V}} \right) \).
  
  Taking \(\delta = \frac{1}{T^2MK^2}\), this expression can be simplified to:
  \[
  \mathcal{R}_T = O\left( \log(TMK) \cdot \left[ \frac{K^2 M \cdot \max(M, K - V)}{\Delta} \right] \right).
  \]
\end{proof}

    \begin{theorem}[Instance-Independent Regret of \textsc{A-CAPALLE}]
      \label{thm:instance_independent_regret}
      If \(\delta = \frac{1}{T^2MK^2}\), the total cumulative regret of A-CAPALLE up to time \( T \) is bounded by
      \[
      \mathcal{R}_T = \tilde{O}\left(
      MK\sqrt{K(K-V)T \max\{ M, K-V\} }
      \right).
      \]
      where the \( \tilde{O} \) notation hides logarithmic factors in \( T, M, K \).
      \end{theorem}
      
      \begin{proof}
      We begin by recalling the instance-dependent regret bound for a single round of partitioning, which combines the cost of detecting a valid separation (Phase 2) and estimating capacities within the top set (Phase 3). Let \( V \) be the minimal number of arms that can accommodate all \( M \) players. The regret incurred during the first partition is given by:
      \[
      \mathcal{R}_{\text{part}}(\varepsilon_1)
      = O\left(
      \log(TMK) \cdot \left[
      \frac{M K (K - V)}{\varepsilon_1} +
      \frac{M^2 K}{\varepsilon_1}
      \right]
      + (K - V)  \varepsilon_1  T M 
      \right),
      \]
      where \( \varepsilon_1 \) is the resolution to which the algorithm distinguishes between arm groups.(i.e. \(\varepsilon_1 = \max_i\gapi{i}\)) The first two terms represent the sample complexity of separating arms and estimating capacities; the last term represents the regret incurred from pulling up to \( (K - V) \) indistinguishable suboptimal arms per round across all \( M \) players. Notice that regret only happens if player are pulling arms not belonging to the top arm sets.
      
      To balance the estimation and suboptimal regret terms, we choose \( \varepsilon_1 \) to minimize the total regret. Setting the derivative of the regret expression to zero gives the optimal value:
      \[
      \varepsilon_1^\star = \sqrt{
      \frac{
      \log(TMK) \cdot ( M K \max\{(K - V) ,M\}
      }{
      (K - V) \cdot T M 
      }
      }.
      \]
      
      Substituting \( \varepsilon_1^\star \) into the regret expression yields:
      \[
      \mathcal{R}_{\text{part}} = O\left(
      \sqrt{
      T \cdot \log(TMK) \cdot M K (K - V) \cdot \max(M, K - V)
      }
      \right).
      \]
      
      The recursive partitioning process may repeat at most \( K - 1 \) times, since each round fixes at least one arm. Thus, multiplying the regret per partition by \( K \) gives the total regret:
      \[
      \mathcal{R}_T = O\left(
      M K\sqrt{T K\log(TMK)(K - V)} \cdot \max\{ \sqrt{M}, \sqrt{K - V} \}
      \right).
      \]
      
      Finally, we get the following instance-independent regret bound:
      \[
      \mathcal{R}_T = \tilde{O}\left(
      MK\sqrt{K(K-V)T \max\{ M, K-V\} }
      \right).
      \]
      \end{proof}

\subsection{Relaxation of Assumption~\ref{assumption:distinct-means}} \label{sec:relax1.1}  

In the main text, we assumed that all arm means are distinct:  
\[
\mu_{\sigma(1)} > \mu_{\sigma(2)} > \cdots > \mu_{\sigma(K)} ,
\]  
which guarantees a unique permutation \( \sigma \) of the arms. This assumption simplifies notation and exposition but is not essential. Our algorithm continues to function correctly even when multiple arms share the same mean reward. In such cases, the regret analysis requires a slight refinement.

Let  
\[
\mu_{\sigma(1)} \leq \mu_{\sigma(2)} \leq \cdots \leq \mu_{\sigma(K)} ,
\]  
and partition the arms into groups \( S_1, S_2, \dots, S_m \), where all arms within each group share the same mean reward. The groups are ordered by increasing mean, so that  
\[
\mu_{S_i} < \mu_{S_j} \quad \text{whenever } i < j .
\]  

Recall that \(V\) denotes the number of top arms required to accommodate all players. If the \(V\)-th best arm belongs to group \(S_i\), then the regret bound is governed by the separation between groups, rather than the adjacent gaps between individually ranked arms.  In particular, the critical quantity that appears in the regret bound is  
\[
\min\!\left\{ \mu_{S_i} - \mu_{S_{i-1}}, \; \mu_{S_{i+1}} - \mu_{S_i} \right\},
\]  
which captures the smallest margin separating group \(S_i\) from its neighboring groups. When \(S_i\) is the top or bottom group, only the existing side of the minimum is relevant.  

Thus, when tied means occur, our regret analysis naturally extends by treating groups of equal-mean arms as a single unit. The results remain valid, with the critical gaps defined between groups rather than between individual arms.

\section{Algorithm Pseudocode and Protocol Details}
\label{sec:algorithm_appendix}
\subsection{Main Algorithm and Pseudocode for All Phases}

Before delving into the algorithmic details, we first provide a visual illustration of the simple Round Robin strategy, shown in Figure~\ref{fig:single_rr}.

\begin{figure}[H]
  \centering
  \includegraphics[width=0.35\linewidth]{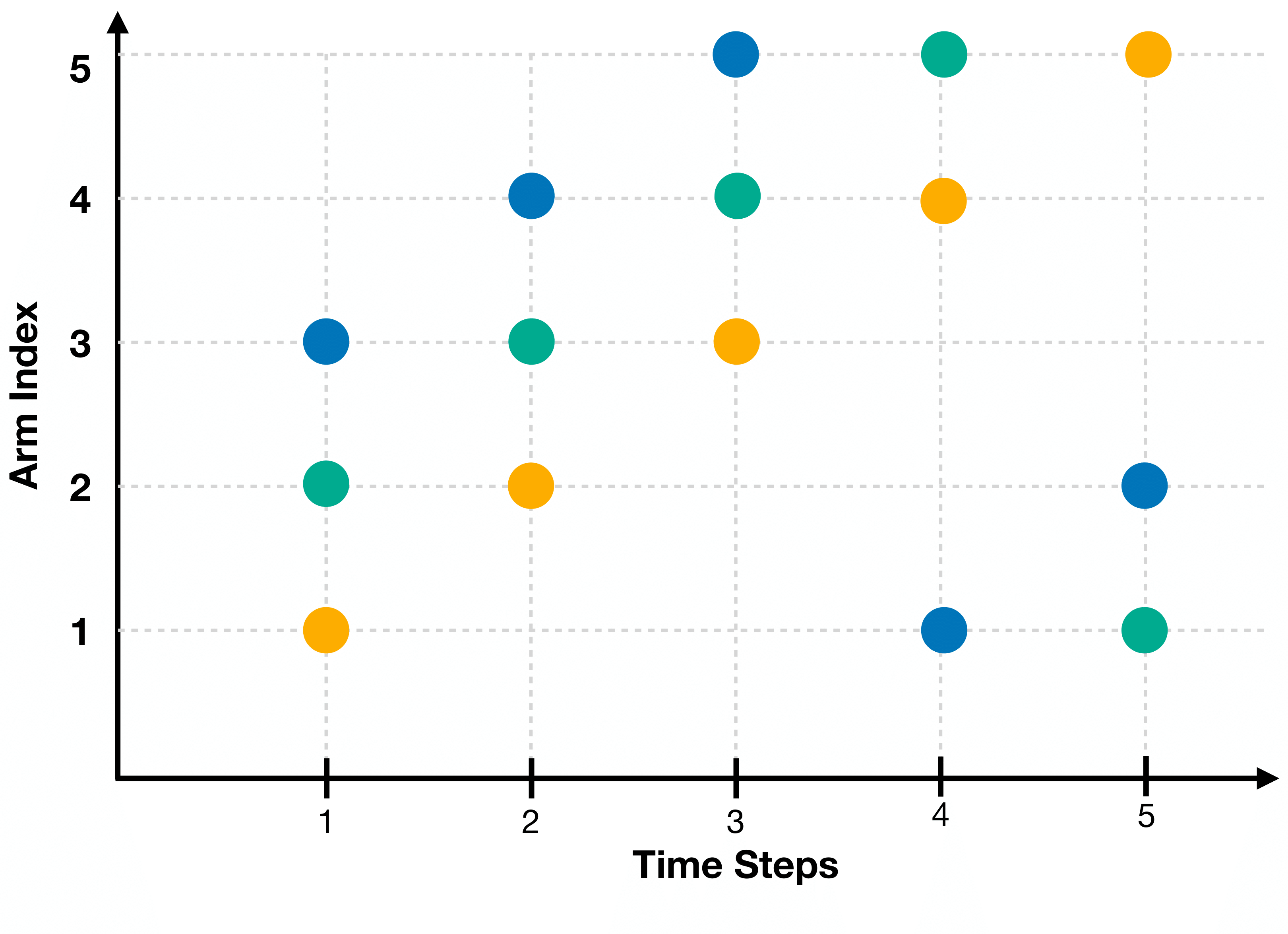}
  \caption{
    Illustration of a simple Round Robin scheduling strategy with 3 players and 5 arms. 
    The x-axis denotes time steps and the y-axis indicates arm indices. 
    Each player cycles through the arms in order: if a player pulls arm \(i\) at time \(t\), 
    they pull arm \(i+1\) in the next round. 
    A round completes when every player has pulled all arms once.
  }
  \label{fig:single_rr}
\end{figure}

\subsubsection*{Phase 1: Communication Arm Selection and Signaling}
In Phase 1, the primary objectives are to employ the deterministic GroupedRoundRobin protocol, which synchronizes the actions of the players, to identify the communication arm and signal it through intentional collision patterns. Synchronization of player actions is achieved using inflated confidence bounds and predefined regular checkpoints, denoted as \(\mathcal{T}_{\text{test}}\).

\(\mathcal{T}_{\text{test}}\) is defined as a set of specific time points where the players check for synchronization and communication signals. These checkpoints are strategically chosen to ensure that all players are aligned in their actions and can detect the communication arm effectively. 
We define checkpoints \(\mathcal{T}_{\text{test}}\) as time steps where the base-9 floor function value changes:
\(
\mathcal{T}_{\text{test}} = \left\{ t \mid \left\lfloor \frac{2}{B(N_t(\cdot,1))^2} \right\rfloor_9 > \left\lfloor \frac{2}{B(N_{t-1}(\cdot,1))^2} \right\rfloor_9 \right\}
\)
, where for any \(x
\in\mathbb{R}^+\), the base-9 floor function is defined as:
\(
\left\lfloor x \right\rfloor_9 = 9^{\alpha} \text{ for the unique } \alpha \in \mathbb{N} \text{ satisfying } 9^{\alpha} \leq x < 9^{\alpha+1}
\).

 In order to create an intentional collision pattern for signaling, players must first independently select a subset of arms, denoted by \( \Sp \), which contains potential communication arms. They then estimate the capacities of all arms in this candidate set. After identifying \( \Sp \), the listeners wait for the subsequent checkpoints. If all players detect that an arm \( \nu \in \Sp \) exhibits an unusual sequence of consecutive zero rewards, despite the arm's capacity not being exceeded, they infer that this arm is the communication arm selected by the coordinator. This is because, while the other players continue to follow a grouped round-robin schedule, the coordinator exclusively pulls the communication arm to reduce its available capacity by 1 in the following Grouped Round Robin Session.

    \paragraph{Coordinator Protocol (Player 1)}
    The main tasks for the coordinator are outlined in the bullet points below. 
    \begin{itemize}
        \item Identify candidate communication arms $\mathcal{S}^1$ through exploration from Group Round Robin Scheduling (lines 4-7 in Algorithm~\ref{alg:coordinator_protocol_phase1}).
        \item Wait until \(t_\text{comm}\), and select the communication arm $\Carm$ with the highest reward and capacity less than $M$ (lines 10-13 in Algorithm~\ref{alg:coordinator_protocol_phase1}).

        \item Signal $\Carm$ through deliberate collision patterns starting from \(t_\text{comm}\)  (lines 14-16 in Algorithm~\ref{alg:coordinator_protocol_phase1}).
    \end{itemize}
    At line 16 in Algorithm~\ref{alg:coordinator_protocol_phase1}, the coordinator will switch to playing a standard UCB algorithm on $\mathcal{S}^1$ if no valid arms in $\mathcal{S}^1$ can be selected as the communication arm. This is because the best arm \(\arm{1}^*\) is guaranteed to be in \(\mathcal{S}^1\) with a probability of at least \(1-\delta\). If no valid arm can be selected, it means that the best arm must have a capacity of at least \(M\); therefore, playing UCB will lead to no collisions over \(\mathcal{S}^1\).

Before selecting the communication arms, the coordinator needs to know the capacities of the arms in \( \mathcal{S}^1 \). Fortunately, she already has sufficient information to estimate these capacities by Lemma~\ref{lemma:capacities in Sp}. An important question we need to address in detail is: how does the coordinator determine the capacities of the arms in \( \mathcal{S}^1 \)?

Define \( \mathbf{r}_{a,t}^{(\psi)} \) to be the sequence of rewards collected by the coordinator up to time step \( t \) from arm \( a \) during \( \psi \)-Grouped Round Robin, where players are grouped into groups of size \( \psi \).
For any  \( a \in \mathcal{S}^1 \), by observing that the sequence \( \mathbf{r}_{a,t_0^1}^{(\psi)} \) consists entirely of zeros for some \( \psi \in [M] \), the capacity \( C_a \) can be determined as
\[
C_a = \min \left\{ \psi \in [M] : \mathbf{r}_{a,t_0^1}^{(\psi)} \text{ is a sequence of all zeros} \right\} - 1.
\]
That is, before time \( t_0^1 \), if the coordinator observes only zero rewards when pulling arm \( a \) in all groups of size \( \psi' \in [M] \), and \( \psi \) is the smallest of such group sizes, then the coordinator infers that the capacity of arm \( a \) is \( C_a = \psi - 1 \).

    \paragraph{Listener Protocol (Players 2 to M)} The main tasks for listeners are outlined in the bullet points below. 
    \begin{itemize}
        \item Identify the set $\mathcal{S}^p$ of  potential communication arms  (lines 1-9 in Algorithm~\ref{alg:NonCoordinator_protocol_phase1}).
        \item Wait for the coordinator's signal of the communication arm (lines 10-26 in Algorithm~\ref{alg:NonCoordinator_protocol_phase1}).
        \begin{itemize}
        \item Decode the communication arm $\Carm$ and communication duration $\omega(12, t_{\text{comm}})$ (lines 17-22 in Algorithm~\ref{alg:NonCoordinator_protocol_phase1}).
        \end{itemize}
    \end{itemize}

    Similar to the protocol of the coordinator, the listener will switch to playing a standard UCB algorithm on $\mathcal{S}^p$ if they do not receive any signal from the coordinator for three consecutive test times. This is because the best arm \(\arm{1}\) is guaranteed to be in \(\mathcal{S}^p\) under \(\Egood\). If no valid arm can be selected, it means that the best arm must have a capacity of at least \(M\).

Line 17 of the algorithm searches for an arm \( \nu_p \in \mathcal{S}^p \) such that its capacity \( C_{\nu_p} \) is known and a valid signal has been received. For any \( \nu_p \in \mathcal{S}^p \), the capacity \( C_{\nu_p} \) is considered known after time step \( t \) if 
\[
n = \min_{\psi \geq 2} N_t(\nu_p, \psi)
\]
is greater than 
\[
\frac{\log(nMK^2/\delta)}{\hat{\mu}_{\nu_p} - B(n)}.
\]
Since Lemma~\ref{lemma:capacities in Sp} guarantees that the capacities of all arms in the set \( \mathcal{S}^1 \) are known to all players at \(t_0^1\), it follows that by time \( t_{\text{comm}} \), every player has determined the capacities of all arms in \( \mathcal{S}^1 \). Thus, they are ready to listen for the signal sent by the coordinator, that is, they are prepared to detect the intentionally disrupted collision pattern caused by the coordinator.

Listeners determine the capacities of arms in the same way as the coordinators as introduced in the last section. For example, before time \( t_{\nu_p} \), which is the time player \(p\) figures out the capacity of arm \(\nu_p\), if the player observes only zero rewards when pulling arm \( a \) in all groups of size \( \psi' \in [M] \), and \( \psi \) is the smallest of such group sizes, then the player infers that the capacity of arm \( a \) is \( C_a = \psi - 1 \).

    \begin{algorithm}     
        \caption{Detect Communication Arm (Player \( p \neq 1 \))}
        \label{alg:NonCoordinator_protocol_phase1}
        \begin{algorithmic}[1]
        
        \State Initialize \( \mathcal{S}^p \gets \emptyset \), \(t_{\text{comm}}\gets \emptyset \)
        
        \While{\( \mathcal{S}^p = \emptyset \)}
            \State Perform \textsc{GroupedRoundRobin} (exploration phase)
            \If{ \( \exists \nu_{p^*} \text{ s.t. } 5B(N_t(\nu_{p^*},1)) > \hat{\mu}_{\nu_{p^*},1}^p(t) \) }
                \State Set \( t_0^p \gets t \)
                \State Add \( \nu_{p^*} \) to \( \mathcal{S}^p \)
                \State Add all \( \nu' \) such that \( \hat{\mu}_{\nu',1}^p(t_0^p) \geq \frac{9}{25} \hat{\mu}_{\nu_{p^*},1}^p(t_0^p) \) to \( \mathcal{S}^p \)
            \EndIf
        \EndWhile
        
        \State Set \( \texttt{no\_signal\_rounds} \gets 0 \)
        
        \While{\( \texttt{no\_signal\_rounds} < 3 \)}
            \State Select next checkpoint \( t'_p\in\mathcal{T}_{\text{test}} \) s.t. \(t'_p \geq t_0^p\) from a predefined set of test times
            \While{\( t < t'_p  \)}
                \State Perform \textsc{GroupedRoundRobin} (waiting for checkpoint)
            \EndWhile
            \State Compute \( \omega(12, t'_p), \forall \nu_p \in \mathcal{S}^p \) based on current observation
        
            \If{\(C_{\nu_p}\) is known and a valid signal is received from some \( \nu_p \in \mathcal{S}^p \) }
            \\ \Comment{\( \nu_p \) gets \( \omega(12, t'_p) \) consecutive zero rewards after \( t'_p \) and \( C_{\nu_p} \) not exceeded}
                \State Set \( \Carm \gets \nu_p \) as the communication arm
                \State Set \( t_{\text{comm}} \gets t'_p \)
                \State \Return \( \Carm \) and  \( \omega(12, t_{\text{comm}}) \)
            \Else
                \State \( \texttt{no\_signal\_rounds} \gets \texttt{no\_signal\_rounds} + 1 \)
                \State \(t_0^p \gets t'_p \)
            \EndIf
        \EndWhile
        \State \textbf{Fallback:} switch to UCB algorithm over \( \mathcal{S}^p \) starting from \( t_0^p \)
        \end{algorithmic}
        \end{algorithm}

\newpage
\subsubsection*{Phase 2: Connectivity Graph Partitioning}
After identifying the communication arm \( \Carm \) in Phase~1, players follow a \textsc{SimpleRoundRobin} schedule to explore all arms. At each checkpoint \( t \in \mathcal{T}_{\text{test}} \), they insert a \textit{signal testing block} for collision detection, which does not update the pull count \( N(\cdot) \). This process continues until the coordinator can separate arms into two clusters with disjoint confidence intervals. At that point, the coordinator uses \( \Carm \) to signal and transmit the top cluster arms using \( K \) bits. Phase~2 may be repeated in future epochs to further refine the partitioning and ultimately identify the optimal \( V \) arms.

\paragraph{Coordinator Protocol:} The protocol for the coordinator is detailed in Algorithm~\ref{alg:phase2_coordinator}. It is important to note that this phase takes in the set of active players \( \mathcal{P} \), the active arm set \( \mathcal{A} \) (which initially is the full arm set \( [K] \), and is updated in the recursive phase), and the minimum signal testing duration \(\omega(12,t_{\text{comm}})\) as input. The coordinator is not restricted to signaling only the top component in the graph \( G_{\text{top}} \); if the top \( V \) arms and their capacities are already known, the coordinator may directly inform all players of the final selected set.

The task of the coordinator can be summarized as follows:
\begin{itemize}
    \item Initialize the connectivity graph \( G \) for the arm set \( \mathcal{A} \) (lines 1-2 in Algorithm~\ref{alg:phase2_coordinator}).
    \item While the number of connected components in \( G \) is not greater than 1:
    \begin{itemize}
        \item Run \textsc{SimpleRoundRobin} schedule while updating the connectivity graph \( G \).
        \item If \( t \in \mathcal{T}_{\text{test}} \), perform a signal testing block for collision detection (line 8 in Algorithm~\ref{alg:phase2_coordinator}). The signal testing block is done over the full arm set \( [K] \), and all players join the signal testing block.
    \end{itemize}
    \item When the number of connected components in \( G \) is greater than 1, the coordinator can signal the start of communication and the top cluster arms using \( \Carm \) (lines 11-20 in Algorithm~\ref{alg:phase2_coordinator}).
\end{itemize}

    \paragraph{Listeners Protocol:} The protocol for the Listeners is detailed in Algorithm~\ref{alg:phase2_listeners}. This algorithm receives the set of active players \( \mathcal{P} \), the active arm set \( \mathcal{A} \) (which initially is the full arm set \( [K] \), and is updated in the recursive phase), and the minimum signal testing duration \(\omega(12,t_{\text{comm}})\) as input. The task of the listeners is to detect the start of communication at every checkpoint \( t \in \mathcal{T}_{\text{test}} \), once they receive the signal from the coordinator, they decode the message about the top cluster arms from the coordinator (lines 8 in Algorithm~\ref{alg:phase2_listeners}).
    \begin{algorithm}[H]
        \caption{Phase 2: Listeners}
        \label{alg:phase2_listeners}
        \begin{algorithmic}[1]
        \State \textbf{Input:} Active players \( \mathcal{P} \), arm set \( \mathcal{A} \), minimum signal testing duration \( \omega(12, t_{\text{comm}}) \)
        \State \( \text{signal} \gets 0 \)
        \While{ \( \text{signal} == 0 \) }
            \If{ \( t \notin \mathcal{T}_{\text{test}} \) }
                \State Run \textsc{SimpleRoundRobin} on \( \mathcal{A} \)
                \State Update statistics for all \( i \in \mathcal{A} \)
            \Else
                \State \(\text{signal, }\mathcal{V_\text{top}} \gets \) \text{Listen} (Algorithm~\ref{alg:listen})
                \Comment{Listen and  Decoded Message about \(\mathcal{V_\text{top}}\)}
            \EndIf
        \EndWhile
        \State \Return \(\mathcal{V_\text{top}}\)
        \end{algorithmic}
        \end{algorithm}

\subsubsection*{Phase 3: Capacity Estimation}
At the end of Phase~2, if the capacities of all arms transmitted by the coordinator are not yet known, the coordinator must first transmit the parameter \( \gamma \), which governs the duration \( \omega(\gamma, t_{\text{first}}) \) used for estimating the capacity of arms in the top cluster. Here, \( t_{\text{first}} \) refers to the checkpoint corresponding to the first partition. As the recursive process continues, this will be updated to \( t_{\text{second}} \), and so on.

It is important to note that since all players participate in the collision testing blocks to help disseminate the coordinator's message, even fixed players are aware of which arms were selected in the latest partition. If any transmitted arm has not yet been processed in Phase~3, all players will proceed to Phase~3 to help the coordinator send information about \(\gamma\) to all the active players. Then, active players start finding capacities, while committed players play a fixed arm.

The main goal of transmitting \( \gamma \) is to ensure synchronized capacity estimation across all active players. Specifically, while all active players have estimated the mean rewards of the active arms up to the same level of accuracy, i.e., 
\[
|\mu_a - \hat{\mu}_a^p| \leq B(N_{t_{\text{first}}}(a,1)) \quad \text{and} \quad |\hat{\mu}_a^{p'} - \hat{\mu}_a^p| \leq 2B(N_{t_{\text{first}}}(a,1)),
\]
where \(p\) and \(p'\) are two different active players, this allows all active players to construct a consistent lower bound for each transmitted arm using 
\[
\gamma \cdot B\left(N_{t_{\text{first}}}(a', 1)\right), \quad \text{where } a' = \arg\min_{a \in \mathcal{A}} \hat{\mu}_a^p
\]

Each player \( p \) can compute the largest \( \gamma_p \in \mathbb{N} \) such that
\[
\gamma_p \cdot B(N_{t_{\text{first}}}(a',1)) < \min_{a \in \mathcal{A}} \hat{\mu}_a^p - B(N_{t_{\text{first}}}(a,1)).
\]
Given the bound \( |\hat{\mu}_a^{p'} - \hat{\mu}_a^p| \leq 2B(N_{t_{\text{first}}}(a,1)) \), it follows that \( |\gamma^p - \gamma^{p'}| < 10 \). Therefore, the coordinator only needs to send the last digit of \( \gamma \) to all active listeners, which can be done through 4 rounds of collision testing to encode one of 10 possible digits.

Afterward, only the active players participate in the Capacity Estimation Protocol, while fixed players continue playing their assigned arms. It is straightforward to verify that the capacity estimation process concludes before the next checkpoint, ensuring all players are ready for the next collision testing block. Once Phase~3 ends, listener players are either assigned to a fixed arm or return to Phase~2 to continue the exploration process.

\subsubsection*{Phase 4: Recursive Coordination}
\begin{enumerate}
    \item \textbf{Recursive Phase Execution}:
    \begin{itemize}
        \item Repeat Phases 2-3 on high-reward cluster
        \item Skip Phase 3 if capacities known
        \item Continue until optimal $V$ arms identified
    \end{itemize}
\end{enumerate}
See Algorithm~\ref{alg:recursion}

\subsection{Collision Signaling Details}
\begin{enumerate}
    \item \textbf{Bit Encoding} is conducted using the following method:
    \begin{itemize}
        \item \textbf{1 bit}: Pull the communication arm for \( \omega(12, t_{\text{comm}}) \) rounds to induce a sufficient number of consecutive zero rewards.
        \item \textbf{0 bit}: Execute Grouped Round Robin scheduling for \( \omega(12, t_{\text{comm}}) \) Session.
    \end{itemize}

\item \textbf{Coordinator Encoding and Sending Signals}:

The following algorithm tell how the coordinator sends a 1-bit to listeners.
\begin{algorithm}[H]
    \caption{CollisionTestingCoordinator}
    \label{alg:collision_testing}
    \begin{algorithmic}[1]
    \State \textbf{Input:} Communication arm \(\Carm\), duration \(w\)
    \For{\(t = 1\) \textbf{to} \(w\)}
        \State Coordinator (Player 1) pulls \(\Carm\)
    \EndFor
    \State \textbf{Output:} Signal detected if consecutive zero rewards are observed
    \end{algorithmic}
  \end{algorithm}

The coordinator uses the following algorithm to encode \( V_{\text{top}} \) to the listeners. For each arm in \( [K] \), the coordinator sends a 1-bit message indicating whether the arm belongs to \( V_{\text{top}} \) to all \( \left\lceil \frac{M}{C_{\Carm}} \right\rceil \) groups. 

To send a 1-bit to one group, the coordinator pulls the communication arm for \( \omega(12, t_{\text{comm}}) \) rounds. To send a 0-bit to one group, the coordinator performs \( \text{RR}(K \setminus \Carm, \omega(12, t_{\text{comm}})) \). 

After a bit is sent to one group, the process repeats for each remaining group.

  \begin{algorithm}[H]
    \caption{EncodeArms}
    \label{alg:encode_top_cluster}
    \begin{algorithmic}[1]
    \State \textbf{Input:} A subset of arms \( G'\), Communication arm \(\Carm\),duration \( \omega(12,t_{\text{comm}}) \)
    \State \(w\gets \omega(12,t_{\text{comm}})\times \lceil \frac{M}{C_{\Carm}} \rceil\)
    \For{\( i = 1 \) \textbf{to} \( K \)}
        \If{\( i \in G'\)}
            \State CollisionTestingCoordinator(\(\Carm, w \))
        \Else
            \State Play  \(\text{RR}(K\setminus \Carm,\omega(12,t_{\text{comm}}))\) for \(\lceil \frac{M}{C_{\Carm}} \rceil\) times
        \EndIf
        
    \EndFor
    \end{algorithmic}
  \end{algorithm}

  \item \textbf{Listener Decoding}:
  
The following algorithm allows each listener (i.e., any player \( p \neq 1 \)) to decode a bit message sent by the coordinator using intentional collisions on the communication arm \( \Carm \). The encoding scheme relies on the observation of consecutive zero rewards caused by capacity violations. The coordinator encodes a 1-bit by pulling the communication arm exclusively for a fixed number of rounds, which results in all other players observing zero rewards due to exceeding the arm’s capacity. Conversely, a 0-bit is encoded by not pulling the communication arm exclusively—thus no unusual sequence of zero rewards is observed.

In this procedure, the listener pulls the communication arm for a predefined number of rounds, denoted by \( \omega(12, t_{\text{comm}}) \), and counts how many of those rounds return zero rewards. If all rewards during this interval are zero (i.e., the number of zero rewards equals \(\omega(12, t_{\text{comm}}) \)), the listener decodes the received bit as 1. Otherwise, the bit is interpreted as 0.

  \begin{algorithm}[H]
    \caption{CollisionTestingListener}
    \label{alg:collision_testing_non_coordinator}
    \begin{algorithmic}[1]
    \State \textbf{Input:} Communication arm \(\Carm\), duration \(\omega(12,t_{\text{comm}})\)
    \State \textbf{Initialize:} Zero reward count \(\text{zero\_count} \gets 0\)
    \For{\(t = 1\) \textbf{to} \(\omega(12,t_{\text{comm}})\)}
        \State Player \( p \neq 1 \) pulls \(\Carm\)
        \State \(\text{zero\_count} \gets \text{zero\_count} + 1\) \textbf{if} zero reward is observed
    \EndFor
    \State \Return 1 \textbf{if} \(\text{zero\_count} = w\), 0 \textbf{otherwise}
    \end{algorithmic}
  \end{algorithm}

In the following algorithm, each listener decodes the message sent by the coordinator that indicates which arms belong to \( V_{\text{top}} \), by repeating Algorithm~\ref{alg:collision_testing_non_coordinator} \( K \) times, once for each arm index.

  \begin{algorithm}[H]
    \caption{DecodeArms}
    \label{alg:decode_top_cluster}
    \begin{algorithmic}[1]
    \State \textbf{Input:}  Communication arm \(\Carm\),duration \( \omega(12,t_{\text{comm}}) \)
    \State \textbf{Initialize:} \( \textbf{Decoded\_message}\gets \emptyset \)
    \For{\( i = 1 \) \textbf{to} \( K \)}
        \If{CollisionTestingNon-Coordinator(\(\Carm, \omega(12,t_{\text{comm}}) \))}
            \State \( \textbf{Decoded\_message}\gets \textbf{Decoded\_message} \cup \{i\} \)
        \EndIf
    \EndFor
    \State \Return \( \textbf{Decoded\_message} \)
    \end{algorithmic}
  \end{algorithm}

  This algorithm is executed by listener players and serves two main purposes. First, it detects whether a signal has been sent by the coordinator, using collision-based signaling mechanism (Algorithm~\ref{alg:collision_testing_non_coordinator}). Second, if a signal is successfully detected, the algorithm proceeds to decode a bit-encoded message that identifies the set \( V_{\text{top}} \) (Algorithm~\ref{alg:decode_top_cluster}). 

  \begin{algorithm}[H]
    \caption{Listen}
    \label{alg:listen}
    \begin{algorithmic}[1]
    \State \textbf{Input:} Number of groups \(\mathcal{N}\), Communication arm \(\Carm\), duration \(\omega(12,t_{\text{comm}})\)
    \State \textbf{Initialize:} Message \(\gets \emptyset\), signal \(\gets\) \textbf{false}
    \For{\( i = 1 \) \textbf{to} \(\mathcal{N}\)}
        \If{Group index = \(i\)}
            \State \textbf{signal} \(\gets\) CollisionTestingNon-Coordinator(\(\Carm, \omega(12,t_{\text{comm}})\))
            
        \Else
            \State Play \(\text{RR}(K\setminus \Carm,\omega(12,t_{\text{comm}}))\)
        \EndIf
    \EndFor
    \If{signal}
        \For{\(i = 1\) \textbf{to} \(\mathcal{N}\)}
            \State Message \(\gets\) Decode(\(\Carm, \omega(12,t_{\text{comm}})  \))
        \EndFor
    \EndIf
    \State \Return signal, Message
    \end{algorithmic}
  \end{algorithm}

\end{enumerate}

\subsection{Relaxing the Assumption $K \geq M$}
\label{sec: M K assumption}

We assume $K \geq M$ in the main text to simplify both the algorithm design and analysis. This restriction allows us to focus on the setting where the implementation of Simple Round Robin and Grouped Round Robin strategies is most intuitive and straightforward.

When $K < M$, the overall structure of the algorithm can still be preserved, but modifications to the round-robin protocols are required. In particular, both the Simple and Grouped Round Robin procedures must be revised to ensure that the sample counts remain balanced across players and arms, i.e.,
\[
N_t^i(\nu, \psi) \;\approx\; N_t^j(\nu', \psi), \quad 
\forall \, i,j \in [M], \; \nu,\nu' \in [K], \; t \in \mathbb{N}.
\]

Such modifications are feasible. For example, one naive adjustment to the simple round-robin protocol is to group players with indices greater than $K-1$ together and assign them a temporary index $K$. The standard round-robin can then proceed with $M$ pulls, treating the merged group as a single player. Alternatively, instead of grouping the higher-indexed players, we could group players with indices $1$ to $K-1$ together with players having indices greater than $2K-2$ next. In this case, the exploration phase would rotate among players indexed from $K$ to $2K-2$, ensuring that each of the $M$ players participates in the exploration phase exactly once per full cycle.
The same principles apply when extending these adjustments to the $\psi$-Grouped Round Robin protocol. By carefully constructing overlapping groups, it is possible to maintain balanced sample counts across all arms and players even when $K < M$. Even this naive adjustment only introduces an additional factor of \(\frac{M}{K}\) into the final regret bound, while the dependence on \(\frac{1}{\Delta}\)  is preserved.

\section{Experiment}
\label{sec:experiment}

\subsection{Experiment in Figure~\ref{fig:regret_vs_horizon}}
The experiment in Figure~\ref{fig:regret_vs_horizon} is conducted under the following Bernoulli arm distribution:: $[0.9, 0.8, 0.2, 0.3, 0.2, 0.1]$.  
In Figure~\ref{fig:regret_vs_horizon} we compare \textsc{A-CAPELLA} with \textsc{Selfish} UCB in a setting with $K=5$ arms, $M=4$ players, uniform arm capacities of $C=2$, and reward gap $\Delta=0.5$.  
For \textsc{Selfish} UCB, each player runs UCB independently with exploration radius $5$, i.e.,
\[
I_a(t) = \hat{\mu}_a(t) + 5\sqrt{\tfrac{2 \log t}{N_a(t)}},
\]
with uniform tie-breaking and one initial pull per arm.  
Each horizon configuration is averaged over $100$ independent runs. 

\[
b_{m,k}(t) = \max_{q \in [0, 1]} 
\left\{ 
q : N_{m,k}(t) \, d\!\left(\hat{\mu}_{m,k}(t), q\right) \le f(t)
\right\},
\]
where
\[
f(t) = \log t + c \log \log t, \quad c \ge 0,
\]
and
\[
d(p, q) = 
p \log \frac{p}{q} 
+ (1 - p) \log \frac{1 - p}{1 - q}.
\]
At every \(t\), we choose the arm with highest KL-index. In our experiments, we set \( c = 3 \). 
Building upon \textsc{Selfish} KL-UCB, the Randomized \textsc{Selfish} UCB introduces Gaussian noise into the KL-index of each arm:
\[
b_{m,k}(t) + \frac{Z_{m,k}(t)}{t},
\]
where \( Z_{m,k}(t) \sim \mathcal{N}(0,1) \). 
Similar to \textsc{Selfish} UCB, neither of these methods provides a theoretical upper bound on regret. 

All experiments were conducted over 100 random seeds, with each player's seed defined as the global seed plus their player index. Consistent with our main results, the two additional baselines exhibit substantially higher variance compared to \textsc{A-CAPELLA}, and their average regret remains larger, especially as the horizon \( H \) increases.

\begin{figure}[ht]
  \centering
  \includegraphics[width=1\linewidth]{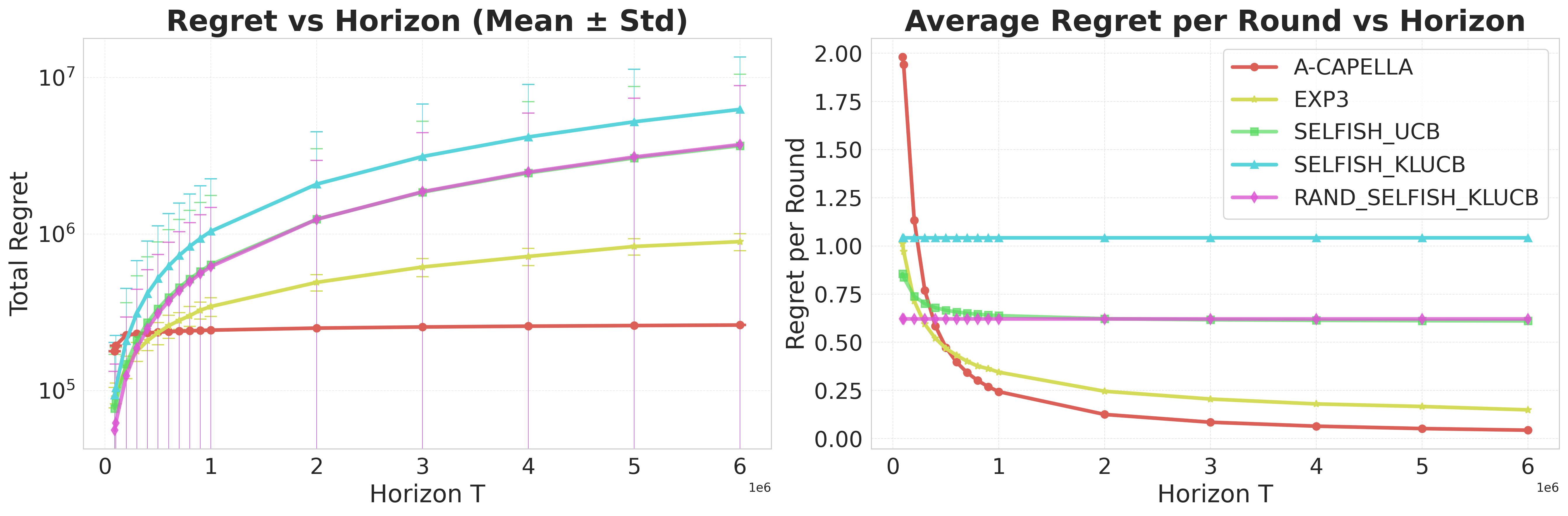}
  \vspace{.1in}
  \caption{
    Regret comparison across algorithms. Error bars represent the standard deviation over 100 random seeds.
  }
  \label{fig:4algos_compare}
\end{figure}

\begin{figure}[ht]
  \centering
  \includegraphics[width=1\linewidth]{regret_boxplots_big_log_exp3added.png}
  \vspace{.1in}
  \caption{
    Box Plot Algorithm Comparison
  }
  \label{fig:box_algo_compare}
\end{figure}

Figure~\ref{fig:4algos_compare} compared 4 algorithms Cumulative Regret vs Horizon and Average Regret vs Horizon. Figure~\ref{fig:box_algo_compare} shows the box plot comparing the same experiment. 
The black line indicates the median for all algorithms, the box spans the 25\% and 75\% quartiles, 
and the scattered points represent outliers.

\subsection{Additional Experiments}
We extend our evaluation with additional experiments to study the effect of problem parameters on regret.  
Unless otherwise specified, we set the horizon to $T = 10^{7}$, assign a uniform capacity of $C=2$ to each arm, and let the expected rewards decrease linearly from $0.9$ to $0.1$ across arms, with a controlled reward gap at the boundary between the $V$-th and $(V{+}1)$-th best arms, i.e., 
\[
\mu_{\sigma(V)} - \mu_{\sigma(V+1)} = \Delta.
\]

\begin{itemize}
    \item Rows 1--3: vary the number of arms $K$ to examine scalability with respect to the action space.  
    \item Rows 4--6: vary the number of players $M$ to test robustness under different levels of competition.  
    \item Rows 7--9: vary the gap parameter $\Delta$ to study its impact on instance-dependent regret.  
\end{itemize}

These experiments collectively demonstrate how regret scales with $K$, $M$, and $\Delta$, complementing the results in Figure~\ref{fig:regret_vs_horizon}.

\begin{table}[ht]
\centering
\caption{Average per-round regret under varying $K$, $M$, and gap $\Delta$. Rows 1--3 vary $K$; rows 4--6 vary $M$; rows 7--9 vary $\Delta$.}
\label{tab:exp_summary}
\begin{tabular}{c c c c}
\toprule
$K$ (Arms) & $M$ (Players) & $\Delta$ & Ave. Regret \\
\midrule
5  & 5  & 0.5 & 0.042 \\
10 & 5  & 0.5 & 0.104 \\
20 & 5  & 0.5 & 0.231 \\
\midrule
20 & 2  & 0.5 & 0.056 \\
20 & 10 & 0.5 & 0.652 \\
20 & 20 & 0.5 & 2.478 \\
\midrule
15 & 5  & 0.5 & 0.167 \\
15 & 5  & 0.3 & 0.790 \\
15 & 5  & 0.1 & 2.000 \\
\bottomrule
\end{tabular}
\end{table}

\subsection{Libraries and computational resources}
\label{sec:Libraries and computational resources}
\paragraph{Libraries used in the experiments.} We set up our experiments using Python 3.13.5 \citep{van1995python}, and made use of the following libraries: NumPy 2.2.6 \citep{harris2020array}, NetworkX 3.5 \citep{hagberg2008exploring},  Seaborn\citep{hunter2007matplotlib}, and Matplotlib\citep{hunter2007matplotlib}. All packages were installed through PyPI. 

\paragraph{Computational resources.} All experiments were conducted on the Shared Computing Cluster (SCC). We primarily used CPU nodes with up to 256 GB of memory, scheduled through the cluster’s queueing system. No specialized GPUs were used. Experiment in Figure~\ref{fig:regret_vs_horizon} was conducted over 10 cpus and to obtain the final results
over 100 seeds we estimated that 24/hours/core are sufficient.

  \section{Modification for Wang et al. 2022a Setting}
  \label{sec:modification for wang}
  We now revise \textsc{A-CAPELLA} to handle the setting where each player only observes the \emph{aggregate reward} on the arm they pull, rather than individual reward feedback, and where only players exceeding an arm's capacity receive zero reward. This feedback model is studied by \citet{wang2022a}. While the overall algorithmic structure remains unchanged, both the capacity estimation procedure and the collision testing routine require adjustment under this model. We assume that the reward distributions for all arms are supported on the interval \([0, 1]\).

  \paragraph{Modified capacity estimation.}
  We replace the overload-detection method, previously based on observing consecutive zero rewards, with one based on detecting a \emph{drop in per-player reward}. Let \( r_{a,t}^{(\psi)} \) be the total reward observed on arm \( a \) at time \( t \) when pulled by a group of size \( \psi \). The per-capita reward is \( r_{a,t}^{(\psi)} / \psi \), and we define the empirical average up to time \( t \) as:
  \[
  \hat{\mu}_{a,\psi}(t) = \frac{1}{t} \sum_{\ell=1}^{t} \frac{r_{a,\ell}^{(\psi)}}{\psi}.
  \]
  Let \( \mu_{a,\psi} = \mathbb{E}[\hat{\mu}_{a,\psi}(t)] \), and suppose \( C_a \) is the capacity of arm \( a \). Then:
  \[
  \mu_{a,\psi} = \mu_{a,1} \cdot \frac{\min\{\psi, C_a\}}{\psi}.
  \]
  Overloading occurs when \( \psi > C_a \), leading to a drop in expected per-player reward. In particular,
  \[
  \mu_{a,C_a+1} - \mu_{a,1} = \frac{\mu_{a,1}}{C_a+1} \ge \frac{\mu_{a,1}}{M},
  \]
  since \( C_a < M \) for the communication arm. This guarantees that the overload-induced drop is at least \( \mu_{a,1}/M \), a quantity that can be estimated once we obtain a reliable estimate of \( \mu_{a,1} \).
  
  Using concentration bounds such as Hoeffding or empirical Bernstein, this drop is detectable with probability at least \( 1 - \delta \) after:
  \[
  O\left( \frac{M^2 \log(1/\delta)}{\mu_{a,1}^2} \right)
  \]
  samples. Accordingly, we set the inflation parameter \( h = 2M\), and redefine the signal-testing checkpoints as:
  \[
  \mathcal{T}_{\text{test}} = \left\{ \left( \frac{4(2M+1)}{2M-1} \right)^{\alpha} : \alpha \in \mathbb{N}^+,\, \left( \frac{4(M+1)}{2M-1} \right)^{\alpha} < T \right\}.
  \]
  Once the following condition is satisfied for any arm \( \nu \):
  \begin{equation}
      h \cdot B(N_{t}(\nu,1)) > \hat{\mu}_{\nu,1}(t),
  \end{equation}
  Then, denote that time as \( t_{0}^1\), the capacity of the communication arm \( \nu \) is known to all players.
  
  \paragraph{Modified collision testing.}
  In this reward feedback model, the collision testing duration must also be adjusted to reflect the noisier observations. Each player listens during a collision testing block of duration:
  \[
  \omega'(t_{\text{comm}}) = \left\lceil \frac{\log(\delta / 4K^2M)}{1\cdot B^2(N_{t_{\text{comm}}}(\cdot,1))} \right\rceil.
  \]
  During this period, players monitor the arms in the candidate set. If a player detects a significant unusual drop in per-player reward on any arm during the block when the arm capacity has not been exceeded, that arm is declared the communication arm. Furthemore, the number of collisions required to transmit a bit is thus given by \( \omega'(t_{\text{comm}}) \).

  \paragraph{Regret analysis} We now analyze the regret of \textsc{A-CAPELLA} under this revised setting. 
  \paragraph{Phase~1.}
Under the event \( \mathcal{E}_{\text{good}} \), detecting a capacity drop requires observing a gap of at least \( \mu_{\sigma_1,1}/M \), which can be reliably identified using empirical concentration bounds after:
\[
O\left( \frac{M^2 \log(1/\delta)}{\mu_{\sigma_1,1}^2} \right)
\]
Grouped Round Robin sessions. 
Each session incurs \( O(2M^2K(\mu_{\sigma_1} - \mu_{\sigma_K}))=O(2M^2K^2\max_i \gapi{i})\) regret. Thus, the total regret in Phase~1 is:
\[
O\left( \frac{(M^4K \cdot \log\left( \log T / \delta \right)}{\mu_{\sigma_1}} \right).
\]
\paragraph{Phase~2 (first partitioning).}
Detecting a disconnection in the connectivity graph using Simple Round-Robin requires:
\[
O\left( \frac{\log(1/\delta)}{\max_i \gapi{i}^2} \right)
\]
rounds, and each round contributes at most \( M(K - V)K \cdot \max_i \gapi{i} \) regret by Lemma~\ref{lemma:simple round robin regret}. The resulting regret from this part is:
\[
O\left( \frac{M(K - V)K \cdot \log(1/\delta)}{\max_i \gapi{i}} \right).
\]

To transmit a bit via collision testing, players observe arms over a duration of
\[
\omega'(t_{\text{comm}}) 
= O\left( \frac{M^2 \log(1/\delta)}{\mu_{\sigma_1}^2} \right).
\]
Each round may incur \( O(M \cdot \mu_{\sigma_1}) \) regret. To transmit \( K + 1 \) bits to all \( M \) players, we need approximately \( \lceil M / C_{\Carm} \rceil \) groups per bit. Thus, the total communication regret is:
\[
O\left( \frac{M^3 K \cdot \log(1/\delta)}{C_{\Carm} \cdot \mu_{\sigma_1}} \right).
\]

Combining both components, the total regret in Phase~2 is:
\[
O\left( \log\left( \frac{\log T}{\delta} \right) \cdot \left(
\frac{M(K - V)K}{\max_i \gapi{i}} + \frac{M^3 K}{C_{\Carm} \cdot \mu_{\sigma_1}} \right) \right).
\]

\paragraph{Phase~3 (first partitioning).}
The dominant cost in Phase~3 arises from capacity estimation using Grouped Round-Robin.  
Under the aggregate feedback model, players must detect a per-player reward gap of at least \( \mu_j / M \), where \( j = \arg\max_i \gapi{i} \), in order to distinguish the top cluster from others. To ensure this separation across all relevant arms, the number of Grouped Round-Robin sessions required is upper bounded by:
\(
O\left( \frac{M^2}{\max_i \gapi{i}^2} \right).
\)
Each session incurs regret up to \( (MK - M + 1)M\mu_{\sigma_1} \), resulting in a total regret for Phase~3 bounded by:
\[
O\left( \frac{(MK - M + 1)M^3\mu_{\sigma_1}  \cdot \log\left( \log T / \delta \right)}{\max_i \gapi{i}^2} \right).
\]

Since the recursive partitioning process repeats at most \( K - 1 \) times, and we set \( \delta = \frac{1}{T^2MK^2} \), the total regret incurred across all recursive partitions is bounded by:
\[
O\left( \log(TMK) \cdot \left[
\frac{M(K - V)K^2 }{\Delta}+\frac{ (MK - M + 1)M^3K\mu_{\sigma_1}  }{\Delta^2}
+ \frac{M^3 K^2}{C_{\Carm} \cdot \mu_{\sigma_1}}
\right] \right),
\]
where \( \Delta = \min\{\gapi{V}, \gapi{V-1}\} \). Suppose \(\arm{K}<\Delta\), then \(\arm{1}<K \Delta\). The total regret can be simplied to 
\[
O\left( \log(TMK) \cdot \left[
\frac{M^4K^3 }{\Delta}
\right] \right),
\]
We emphasize that this bound is conservative: it uses worst-case regret estimates from Simple and Grouped Round Robin under the assumption that \emph{all players} pulling an overloaded arm receive zero reward. In the aggregate feedback model, however, only the \emph{overloaded players} receive zero reward, while players within capacity still receive non-zero feedback. As a result, the actual regret in practice is significantly smaller than this upper bound.

\end{document}